\documentclass{article}

\usepackage[preprint]{neurips_2024}

\usepackage[utf8]{inputenc} 
\usepackage[T1]{fontenc}    
\usepackage{hyperref}       
\usepackage{url}            
\usepackage{booktabs}       
\usepackage{amsfonts}       
\usepackage{nicefrac}       
\usepackage{microtype}      
\usepackage{xcolor}         

\usepackage{algorithmic}
\usepackage{pifont}
\usepackage[linesnumbered, ruled]{algorithm2e}

\usepackage{microtype}
\usepackage{graphicx}
\usepackage{subfigure}
\usepackage{booktabs} 
\usepackage{bm}
\usepackage{enumitem}

\usepackage{diagbox}
\usepackage{multirow}

\usepackage{wrapfig}

\usepackage{amsmath}
\usepackage{amssymb}
\usepackage{mathtools}
\usepackage{amsthm}

\usepackage[capitalize,noabbrev]{cleveref}

\hypersetup{hidelinks,colorlinks=true,
linkcolor=black,anchorcolor=black,citecolor=darkgray}

\theoremstyle{plain}
\newtheorem{theorem}{Theorem}[section]

\newtheorem{lemma}[theorem]{Lemma}

\theoremstyle{definition}

\newtheorem{assumption}[theorem]{Assumption}
\theoremstyle{remark}
\newtheorem{remark}[theorem]{\bf Remark}

\usepackage[textsize=tiny]{todonotes}

\title{Sequential Federated Learning in Hierarchical Architecture on Non-IID  Datasets}

\author{%
  Xingrun Yan\\
   Beijing Institute of Technology\\
   \texttt{3220221473@bit.edu.cn} \\
   \And
  Shiyuan Zuo \\
   Beijing Institute of Technology\\
  \texttt{zuoshiyuan@bit.edu.cn} \\
   \And
   Rongfei Fan\\
   Beijing Institute of Technology\\
   \texttt{fanrongfei@bit.edu.cn} \\
   \And
      Han Hu\\
   Beijing Institute of Technology\\
   \texttt{hhu@bit.edu.cn} \\
   \And
   Li Shen \\
   Sun Yat-Sen University\\
   \texttt{shenli6@mail.sysu.edu.cn} \\
   \And
   Puning Zhao\\
   Zhejiang Lab\\
   \texttt{pnzhao@zhejianglab.com} \\
    \And
   Yong Luo\\
   Wuhan University\\
   \texttt{luoyong@whu.edu.cn} \\
}

\begin{document}

\maketitle

\begin{abstract}
In a real federated learning (FL) system, communication overhead for passing model parameters between the clients and the parameter server (PS) is often a bottleneck.
Hierarchical federated learning (HFL) that poses multiple edge servers (ESs) between clients and the PS can partially alleviate communication pressure but still needs the aggregation of model parameters from multiple ESs at the PS.
To further reduce communication overhead, we bring sequential FL (SFL) into HFL for the first time, which removes the central PS and enables the model training to be completed only through passing the global model between two adjacent ESs for each iteration, and propose a novel algorithm adaptive to such a combinational framework, referred to as Fed-CHS. Convergence results are derived for strongly convex and non-convex loss functions under various data heterogeneity setups, which show comparable convergence performance with the algorithms for HFL or SFL solely.
Experimental results provide evidence of the superiority of our proposed Fed-CHS on both communication overhead saving and test accuracy over baseline methods.
\end{abstract}
\section{Introduction}
\label{sec:intro}

Recently, federated learning (FL) stands out as an emerging distributed ML approach that can alleviate the privacy concern of clients via multiple rounds of local training \citep{9796935,MA2022244,10251949}.
In each round, participating clients first utilize their local data and the received FL global model in the last round to update their FL local models respectively, and then transmit the parameters of their most up-to-date local models to a parameter server (PS) \citep{8664630,9835537}. Hereafter, the PS aggregates these uploaded local model parameters to have a new one and broadcasts it to each participating client. 
Convergence of global model can be promised in such a framework and no raw data is transmitted between any participating client and the PS \citep{NEURIPS2021_64be20f6,NEURIPS2022_449590df,pmlr-v202-song23e}.

In traditional FL architecture, as the number of participating clients grows, it is very pricey to sustain a dedicated communication link between every client and the PS for global model updating, considering the data size of model parameter nowadays is generally enormous
\citep{8885054,pmlr-v119-rothchild20a,9488679,9660377,9014530}. 
To overcome this challenge, some efforts are made to reduce the number of bits to offload for each communication hop, e.g., through pruning the neural network model \citep{NEURIPS2021_6aed000a,ijcai2022p385} or compressing the model parameter \citep{Bibikar_Vikalo_Wang_Chen_2022}, etc. 
Differently, some other works try to save the number of communication hops by designing new aggregation mechanisms with a concern of network topology.
\textit{Hierarchical Federated Learning} (HFL) \citep{9054634,9207469,9479786} is such a kind of example that emerges recently.
In HFL, massive clients are broadly distributed and have to get through the PS via an edge server (ES) in the vicinity. Dividing the set of clients into multiple clusters, within each of which there is an associated ES \citep{9148862,10261420,9699080}.
In any round of local model aggregation, each ES first aggregates the local models from the clients in its cluster and then pushes the aggregated one to the PS to generate the global model parameter \citep{yang2021h,9665214,9809926,9834296}.
Through this operation, the communication hops for offloading each client's specific local model parameter from the associated ES node to the central PS are waived \citep{9546457,9685644}.

Even with HFL, 
the round-by-round aggregations of combined model parameters from every ES at the PS is still a communication-heavy task \citep{9488756}.
This situation is even worse when the group of ESs and the PS are in a highly dynamic network, weakly interconnected, or subject to network topology mismatch, i.e., the topology of ESs and the PS is not in a stable and star shape \citep{9829367,10364357}.
To overcome this challenge, the idea of \textit{sequential FL} (SFL) that removes PS from the FL system and allows global model parameters updated among participating clients sequentially, can serve as an enhancement, but has not been brought into HFL in literature.
To leverage SFL in the framework of HFL, aggregated model parameter is updated sequentially among ESs,
which exempts the parallel offloading of each ES to the PS \citep{pmlr-v162-sun22b,10044204}.  
The combination of HFL and SFL is adaptive to a broad range of network topologies. 
Some exemplary applications include, but are not limited to, Internet of Vehicles (IoV) and Integrated Low-earth Orbit (LEO) Satellite-Terrestrial Network, whose detailed descriptions are listed in \cref{apx:app}.

In addition, datasets from disjoint participating clients may be non-independent and identically distributed (non-IID) (or say be subject to data heterogeneity) due to diverse geographical environments and client characteristics \citep{10261420,NEURIPS2022_449590df}. 
Non-IID datasets bring new challenges into convergence proving for training algorithms and have been a vital issue to address in the traditional framework of FL, HFL, or SFL \citep{pmlr-v162-sun22b,10261420,pmlr-v202-song23e}.
In our combinational framework of SFL and HFL, the factor of non-IID cannot be ignored either as the HFL structure enables the coverage of a broader range of clients.
{
How to promise convergence in such a situation is still an open problem, and cannot be answered based on existing related works that cares SFL or HFL sorely \citep{NEURIPS2023_74088c68,9834296}.}



In this paper, to save communication hops and overcome the data heterogeneity that are widely existed, we investigate the realization of \underline{fed}erated learning in the \underline{c}ombinational framework of \underline{H}FL and \underline{S}FL with non-IID data sets, and propose a novel aggregation algorithm named Fed-CHS. 
In Fed-CHS, iterative training is performed cluster-by-cluster. For the currently active cluster, each client in it only needs to interact with the associated ES node, who then generates the most up-to-date global model parameter from these interactions and pushes the generated model parameter to the ES node in a neighbor cluster for next round of iteration. 
The selection of next passing ES node is realized based on a deterministic and simple rule. 
Rigorous convergence analysis is expanded when the loss function is strongly convex and non-convex, under various setup of data heterogeneity.

\textbf{Contributions:}
Our main contributions are three-fold: 
\begin{itemize}[leftmargin=*]
    \item {\bf Algorithmically}, we investigate the combinational framework of SFL and HFL for the first time and propose a new aggregation algorithm Fed-CHS to be adaptive to it.
    {Fed-CHS offers a communication overhead saving method compared with existing methods for HFL architecture or conventional FL framework and is totally different from existing communication-efficient methods that merely save the number of bits for each communication hop.}
    {Furthermore Fed-CHS is general to network topology, especially when the topology is highly dynamic or not in a star shape}.
    \item {\bf Theoretically}, we explore the convergence performance of Fed-CHS when the loss function is strongly convex or non-convex. Specifically, Fed-CHS can converge at a linear rate with $T$ and a rate $O(1/K^{q-1})$ for $q\geq 2$ with $K$, which defines the rounds of training iteration among clusters and within a cluster respectively, for strongly convex loss function, and a rate of $\mathcal{O}(1/T^{1+ q_1 -q_2})$ with $1 +q_1 > q_2$ and $q_2 \ge q_1$ for non-convex loss function under general data heterogeneity. 
    Additionally, when data heterogeneity fades away partly or completely, Fed-CHS can further achieve zero optimality gap for strongly convex loss function, and stationary point for non-convex loss function, respectively, while keeping the convergence rate unchanged. These convergence results are comparable with the algorithms for HFL or SFL solely.
    \item {\bf Numerically}, we conduct extensive experiments to compare Fed-CHS with baseline methods. The results verify the advantage of our proposed Fed-CHS in test accuracy and communication overhead saving.
\end{itemize}

\section{Related Work in Brief}

FL was firstly introduced in \citet{mcmahan2017communication}  and then evolves to be HFL or SFL for the sake of relieving communication overhead.
Optimality gap and convergence rate are the two most analytical performance metrics in literature under diverse assumptions on  loss function's convexity, especially for non-IID dataset, which are surveyed briefly in the following. More details about the surveyed literature are given in \cref{app:survey}.


{\bf Hierarchical Federated Learning (HFL)}
In the framework of HFL, \citet{tu2020network} concerns limited computation capability at clients, \citet{9705093} cares a special network connection topology like device-to-device (D2D) , and \citet{9699080} takes communication bottleneck into account by allowing ES currently having bad link condition with the PS to delay the uploading of model parameter. However, \citet{9699080} still requires each ES to offload aggregated model parameter to the PS. Heavy parameter offloading task is only postponed rather than diminished in \citet{9699080}.

{\bf Sequential Federated Learning (SFL)}
With regard to SFL, the {\it Next Passing Node} that is going to receive the most up-to-date global model parameter, may be selected by following a fixed order or randomly. 
When the next passing node is chosen by following a fixed order, ring topology has to be imposed \citep{wang2022efficient,9956084}, which can hardly accommodate to various network topology in real application. 
When the next passing node is selected in a random way, more computation overhead may be brought in for evaluating the Lipschitz constant as required in \citet{8778390} or the value function for working about an multiple armed bandit (MAB) problem in \citet{10044204}. 

{\bf Convergence Guarantees for Federated Learning}
In terms of convergence analysis concerning data heterogeneity, for HFL, convex loss function is mainly concerned \citep{tu2020network,9705093,9699080}. 
Some literature, including \citet{tu2020network,9699080}, cannot achieve zero optimality gap but is able to converge at linear rate, for a general convex \citep{tu2020network,9699080} or strongly convex \citep{9699080} loss function.
\citet{9705093} can achieve zero optimality gap but at a convergence rate of $\mathcal{O}(1/T)$.
However, the zero optimality gap in \citet{9705093} is build on some special assumption on quantitative relationship between instant gradient of loss function and system parameter. 
While, for SFL, the aforementioned \citet{9050511,10044204} have tried to overcome data heterogeneity. 
The loss function is assumed to be strongly convex in \citet{9050511,10044204}.
\citet{9050511} can achieve convergence at rate $\mathcal{O}(1/T^{1-q})$ for $q \in (0, 1/2)$, while \citet{10044204} is able to further achieve zero optimality gap at at rate $\mathcal{O}(1/\sqrt{T})$, which relies on the holding of a special condition between local model parameter and local gradient in the stage of initialization. 

In summary, as the architecture evolves from conventional FL, the results of which is given in \cref{app:survey}, to HFL and SFL, i.e., from simple to complex, it becomes harder and harder to get ideal results for a general convex or non-convex loss function in front of heterogeneous dataset.
In this situation, deriving analytical results under the combinational framework of HFL and SFL is even more challenging.

\section{Fed-CHS}
\label{sec:proposed model}
In this section, we present Fed-CHS. We first elaborate the problem setup under the hybrid framework of SFL and HFL.

\subsection{Problem Setup}

{\bf FL learning task:} 
Consider a FL system with $N$ clients. These $N$ participating clients compose the set of $\mathcal{N} \triangleq \{1, 2, ..., N\}$.
For any participating client, say $n$th client, it has a local data set $\mathcal{D}_n$ with $D_n$ elements. The $i$th element of $\mathcal{D}_n$ is a ground-true label $z_{n,i} = \{{x}_{n,i}, {y}_{n,i}\}$.
The ${x}_{n,i} \in \mathcal{R}^{\text{in}}$ is the input vector and ${y}_{n,i} \in \mathcal{R}^{\text{out}}$ is the output vector.
By utilizing the data set $\mathcal{D}_n$ for $n\in \mathcal{N}$, 
we aim to train a $d$-dimension vector ${w}$ by minimizing the following loss function
\begin{equation}
F({w}) \triangleq  \frac{1}{D_A} \sum_{n\in \mathcal{N}} \sum_{z_{n,i}\in \mathcal{D}_n} f({w}, z_{n,i})
=  \sum_{n \in \mathcal{N}} \frac{D_n}{D_A} f_n({w})
\end{equation}
where $f({w}, z_{n,i})$ is the loss function to evaluate the error for approximating  ${y}_{n,i}$ with an input of ${x}_{n,i}$ and a selection of ${w}$, $D_A$ is defined as $D_A \triangleq \sum_{n \in \mathcal{N}} D_n$, and $f_n({w})$ represents the local loss function of $n$th client, defined as
\begin{equation}
f_n({w}) \triangleq \frac{1}{D_n} \sum_{z_{n,i} \in \mathcal{D}_n} f({w}, z_{n,i}), \forall n \in \mathcal{N}.
\end{equation}
For the ease of presentation in the following, with a randomly set $\xi_{n} \subseteq \mathcal{D}_n$, we define
\begin{equation}
f(w, \xi_n) \triangleq \frac{1}{|\xi_{n}|} \sum_{z_{n,i} \in \xi_n} f({w}, z_{n,i}), \forall n \in \mathcal{N}.
\end{equation}

Denoting $\gamma_n$ as $\gamma_n = {D_n}/{D_A}$ for $n \in \mathcal{N}$, the training task can be reformulated as the following optimization problem
\begin{equation} \label{e:prob_formu}
	\min_{w\in\mathcal{R}^d}\left\{F\left(w\right)\triangleq \sum_{n=1}^N\gamma_n f_n\left(w\right)\right\}
\end{equation}

Accordingly, the gradient of loss function of  $f_n(w)$, $f(w, \xi_n)$ with respect to $w$ can be written as $\nabla f_n({w})$ or $\nabla f(w, \xi_n)$, and the gradient of global loss function $F(w)$ with respect to $w$ is denoted as $\nabla F(w)$.  

In a traditional framework of FL, to find the solution of the problem in \cref{e:prob_formu}, a PS is obligated to interact with the involved clients iteratively round by round by exchanging the loss function's gradient or model parameter. This will lead to extremely heavy burden considering the model parameter is usually with high dimension. A FL training framework that can save the exchanging of loss function's gradient or model parameter is highly desirable.

{\bf Hybrid framework of SFL and HFL:}


\begin{wrapfigure}{r}{0.5\textwidth}
\centering
\includegraphics[width=0.5\textwidth]{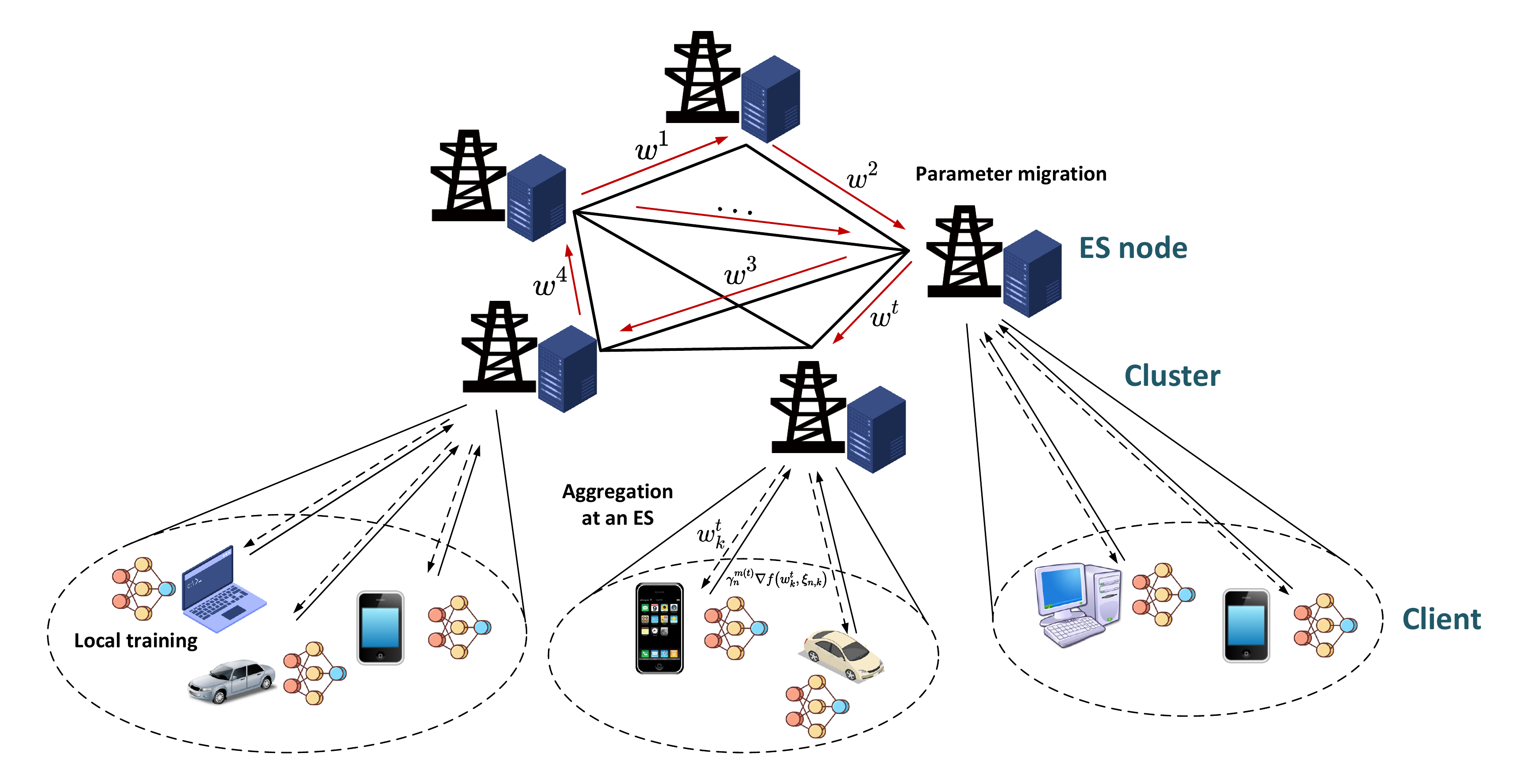}
\caption{The framework of SFL in hierarchical architecture. In this architecture, clients are divided into multiple clusters, each of which is managed by one ES. For each step of iteration, model parameter is firstly updated within one cluster through multiple interactions between the ES and the associated clients, and then migrated to a neighbor ES (cluster) for next step of iteration. 
}
\label{fig:architecture}
\end{wrapfigure}

In this work, we adopt a  communication-efficient hierarchical architecture.
As shown in Fig. \ref{fig:architecture}, the group of $N$ clients are divided into $M$ clusters, which compose the set $\mathcal{M} \triangleq \{1, 2, ..., M\}$.
For $m$th cluster, the set of associated clients are denoted as $\mathcal{N}_m$, and there is one ES that can interact with the clients in set $\mathcal{N}_m$, which is denoted as $m$th ES node.
The group of $M$ ES nodes are connected through communication links.
In previous HFL works, to approximate the training process of traditional FL and to fulfill the advantage of communication-efficiency for this hierarchical architecture, in $t$th round, the local gradients $\nabla f_n({w}^{t})$ or $\nabla f(w^t, \xi_n)$ for $n\in \mathcal{N}_m$ can be firstly aggregated at $m$th ES to obtain an aggregated model parameter $w_{m}^t$ for every $m\in \mathcal{M}$, 
and then form to be a global $w^{t+1}$ through aggregating the set of $\{w_{m}^t | m \in \mathcal{M}\}$ at a PS.
Once a global $w^{t+1}$ is generated, it will be broadcasted to every client in each cluster through the associated ES.
In such a procedure, the communication burden of uploading every specific $\nabla f_n({w}^{t})$ or $\nabla f(w^t, \xi_n)$ for $n\in \mathcal{N}_m$ via $m$th ES to the PS is saved. 
Bursty connecting requests from each client to the PS is also evaded.  

To achieve better communication-efficiency,  by referring to the framework of SFL, we do not assume the existence of PS and waive all the ES's from aggregating $\{w_{m}^t | m \in \mathcal{M}\}$ to be a global one like $w^{t+1}$ in any round $t$.
To be specific, for round $t$, only one selected ES, say $m(t)$th ES, first broadcasts its received most up-to-date model parameter by round $(t-1)$, denoted as $w^{t}$ for the ease of presentation, to all the clients in its cluster, i.e., $m(t)$th cluster, and then aggregates the local gradients in its cluster, i.e., $\nabla f_n({w}^{t})$ or $\nabla f(w^t, \xi_n)$ for $n \in \mathcal{N}_m$ by some rule, to obtain an aggregated model parameter $w^{t+1}$. 
At the ending stage of round $t$, $m$th ES selects a new ES to work as the model parameter broadcaster and aggregator in round $(t+1)$, say $m(t+1)$th ES, and push the most up-to-date model parameter $w^{t+1}$ to $m(t+1)$th ES. 

{\bf Design goal:} Under the above combinational framework of SFL and HFL, our goal is to design a proper aggregating rule at ES side and select a suitable $m(t+1)$ for each round $t$, so as to achieve convergent or optimal solution of the problem in \cref{e:prob_formu} as quickly as possible, while being robust to data heterogeneity.

\subsection{Algorithm Description}
To achieve our design goal, we develop Fed-CHS.
In each training round of Fed-CHS, there is only one cluster being active, within which clients and the ES interacts multiple times to leverage the local dataset at these clients to update the global model parameter. With the global model parameter updated, the ES in currently active cluster selects the next passing cluster (and ES) by some rule and then push the most up-to-date global model parameter to it, in preparation for next round of training.
Below, we give a full description of Fed-CHS (see Algorithm \ref{alg:train process}). Its crucial steps are explained in detail as follows.


\begin{itemize}[leftmargin=*]

\item
{\bf Training within in one cluster:}
In $t$th round, after receiving the global model parameter $w^{t}$ that is broadcasted from $m(t)$th ES, $K$ times of interactions between $m(t)$th ES and all the clients in set $\mathcal{N}_m$ is activated. 
In $k$th step, suppose the global model parameter already known by each client at the beginning of this step is $w_k^t$, client $n$ with $n\in \mathcal{N}_{m(t)}$ will updates its local model parameter based on a randomly selected dataset $\xi_{n,k} \subseteq \mathcal{D}_n$, and then return its gradient $\nabla f\left(w^t_k, \xi_{n, k}\right)$ to ES $m(t)$. For ES $m(t)$, it first aggregates these received local gradients with weighting coefficients $\gamma_{n}^{m(t)}$ such that
$\sum_{n\in \mathcal{N}_{m(t)}} \gamma_{n}^{m(t)} = 1$,
then use the learning rate $\eta_k$ to generate $w^t_{k+1}$, and finally broadcasts $w^t_{k+1}$ to every client in its cluster.
Specifically, $w^t_{k+1}$ can be written as 
\begin{equation} \label{e:w_iter}
w^t_{k+1}\!=\!w^t_k-\eta_k\sum_{n \in \mathcal{N}_m(t)} \gamma_{n}^{m(t)}  \nabla f\left(w^t_k, \xi_{n,k}\right), k  \in \mathcal{K} \triangleq \{0, 1, ..., K-1\}.
\end{equation}
With the iterative expression of $w^t_k$ in \cref{e:w_iter}, there is $w^{t+1}_0= w^{t+1} = w^{t}_{K}$. 

\item
{\bf The selection of next passing cluster:}
To select the next passing cluster, i.e., $m(t+1)$th cluster, in case Fed-CHS has run $t$ rounds of iterations, we propose to select $m(t+1)$ from $\mathcal{A}(m(t))$
by the following two-step rule.
{\bf Step 1:} Generate the set $\mathcal{C}(t) = \{m' | m' = \arg \min c(m'), m' \in \mathcal{A}(m(t))\}$. The set $\mathcal{C}(t)$ actually represents the set of ES nodes that is least traversed in last $t$ rounds of iteration.
If $|\mathcal{C}(t)|=1$, suppose $\mathcal{C}(t) = m^{\dag}$, output $m(t)= m^{\dag}$.
If $|\mathcal{C}(t)|>1$, go to Step 2;
{\bf Step 2:} Set $m^{\ddag} = \mathop{\arg \max} \limits_{m' \in \mathcal{C}(t)} D_{A, m'}$, output $m(t+1) = m^{\ddag}$. This operation is to select the cluster with the largest dataset from $\mathcal{C}(t)$. The essence of above two-step rule is to drive the learning process to cover a broader range of dataset or more diversely distributed dataset.
\end{itemize}

\begin{algorithm}[tb] 
\small
	\caption{Fed-CHS Algorithm} 
	\label{alg:train process}
	\begin{algorithmic}[1]
		\STATE {\bfseries Input:} Initial global model parameter $w^{0}$, client set $\mathcal{N}$, ES node set $\mathcal{M}$, the number of global iteration rounds $T$, the number of local iteration rounds $K$. 
		\STATE {\bfseries Ouput:} Updated global model parameter $w^{T}$.
		\STATE {\% \% \bf Initialization}
		\STATE {ES $m$ builds set $\mathcal{N}_m$, calculates $D_{A,m}$, and establish set $\mathcal{A}_{m}$,  for $m\in \mathcal{M}$. Set $M$-dimensional vector $c = (0,..., 0)^T$. Randomly select $m(0)$ from set $\mathcal{M}$ and set $w^0_0 = w^{0}$.}
		\FOR{$t=0$: $T-1$}
		\STATE {\%\% \bf Training within one cluster}
		\FOR{$k=0$:$K-1$}
		\STATE The $m(t)$th ES broadcasts the most up-to-date global model parameter $w^t_k$ to all the clients in set $\mathcal{N}_{m(t)}$.
		\FOR{every client $n \in \mathcal{N}_{m(t)}$ in parallel}
		\STATE Receive the global parameter $w^t_k$. Randomly select dataset $\xi_{n,k}$ and evaluate $\nabla f\left(w^t_k,\xi_{n, \zeta_k}\right)$. Upload $\nabla f\left(w^t_k,\xi_{n, \zeta_k}\right)$ to $m(t)$th ES.
		\ENDFOR
		\STATE {The $m(t)$th ES receives every $\nabla f\left(w^t_k,\xi_{n, \zeta_k}\right)$ for $n\in \mathcal{N}_{m(t)}$ and generates $w_{k+1}^t$ according to \cref{e:w_iter}.}
		\ENDFOR
		\STATE{\%\% \bf The selection of next passing cluster.}
		\STATE {The $m(t)$th ES retrieves the information of $D_{A,m}$ and the value of $c(m)$ from ES node $m \in \mathcal{A}_{m(t)}$, and selects the next passing node ES $m(t+1)$ according to the 2-step operation.}
		\STATE {The $m(t)$th ES sets $w^{t+1}_0=w^{t+1} = w^t_{K}$ and sends the most up-to-date global model parameter $w^{t+1}_0$ to $m(t+1)$th ES.}
		\STATE {The $m(t+1)$th ES updates $c(m(t+1)) = c(m(t+1))+1$.}
		\ENDFOR
		\STATE Output $w^{T}$.
	\end{algorithmic}
\end{algorithm}

{\bf Communication Overhead:}
With regard to communication burden for passing model parameters or gradients between ES nodes and clients, suppose we need $Q$ bits to quantize a model parameter vector or local gradient vector, both of which is a $d$ dimensional vector of floating numbers.
For $T$ round of iterations, suppose the maximal $|\mathcal{N}_m|$ for $m\in \mathcal{M}$ is $N_{\max}$, according to \cref{alg:train process}, the clients need to upload $TKQN_{\max}$ bits at most to ES nodes, reversely the ES nodes broadcast $TK$ times to deliver no exceeding $TKQN_{\max}$ bits to their associated clients, and $TQ$ bits is required to transmitted among neighbor ES nodes.

\section{Theoretical Results}

In this section, we theoretically analyze the convergence performance of Fed-CHS over the non-IID datasets. Due to limited space, necessary assumptions are listed in Section \ref{s:assum}.
Then analytical results of convergence on our proposed Fed-CHS are presented based on assumptions. 
All the proofs are deferred to \cref{apx:convex} and \cref{apx:nonconvex}.

\subsection{For \texorpdfstring{$\mu$}{mu}-strongly convex loss function}
\begin{theorem}
	\label{thm:convex}
With the holding of \cref{ass:smooth}, \cref{ass:convex}, \cref{ass:gradient}, and \cref{ass:heterogeneity}, 
for the set of $\eta_k$ such that $\eta_k \leq \frac{1}{2LK}, \forall k\in \mathcal{K}$,
define
\begin{align}
		& \Delta_{m} \triangleq \sum_{n \in \mathcal{N}_{m}} \gamma_{n}^{m} \left(f_n\left(w^{t}\right)-f_n\left(w^*\right)\right); \tau_{m} \triangleq \sum_{n \in \mathcal{N}_{m}} \gamma_{n}^{m} \left(f_n\left(w^*\right)-f_n^*\right), \forall m \in \mathcal{M}; \beta \triangleq \frac{\mu}{2} \sum_{k=0}^{K-1} \eta_k, 
	\end{align}
where $w^*$ is the global minimizer of loss function $F(w)$ and $f_n^*$ is the minimum of $f_n(w)$,
then there is
	\begin{align}
		\label{eqn:thm1}
		\nonumber
	&	\mathbb{E}\{F\left(w^{T}\right)\}-F\left(w^*\right)   
  \le  \frac{L}{2}\left(1-\beta\right)^{T} \left\|w^0-w^*\right\|^2 +
  \frac{L}{2}\sum_{k=0}^{K-1} 6LK \eta_k^2 \sum_{t=0}^{T-1} \left(1-\beta\right)^{t} \tau_{m(t)}
   \\
  & +\frac{L}{2}\sum_{k=0}^{K-1} 2\eta_k\left(1-2LK\eta_k\right)\left(LK\eta_k-1\right) \sum_{t=0}^{T-1} \left(1-\beta\right)^{t} \Delta_{m(t)} 
    +\frac{L}{2}\sum_{k=0}^{K-1} \frac{2k}{K} \sum_{j=0}^{k-1} \eta_j^2 \sum_{t=0}^{T-1} \left(1-\beta\right)^{t} G^2 \\ 
    & + \frac{L}{2}\sum_{k=0}^{K-1} \left(K\eta_{k}^2+\frac{2k}{K} \sum_{j=0}^{k-1} \eta_j^2\right) \sum_{t=0}^{T-1} \left(1-\beta\right)^{t} {\theta_{m(t)}^2} \nonumber
	\end{align}
\end{theorem}

 \begin{proof}
 	See \cref{apx:convex}.
 \end{proof}

\begin{remark} \label{r:convex}
The bound of optimality gap in \cref{eqn:thm1} reveals the impact of various system parameters on the performance of training model.
To reduce the optimality gap, we need to suppress the right-hand side of \cref{eqn:thm1} as much as possible.
In the following, discussion is unfolded to show how to make the suppression under some specific configuration of $\eta_k$ for $k\in \mathcal{K}$.
\begin{itemize}[leftmargin=*]
\item Empirically, we do not wish $\{\eta_k\}$ to be too small. Hence we first set $\eta_k$ as $\eta_k = \frac{1}{ 2 L K\sqrt{k+1}}$, which satisfies $\eta_k < \frac{1}{2LK}$ for $k\in \mathcal{K}$ naturally.
In this case, $\beta = \frac{\mu}{2} \sum_{k=0}^{K-1} \eta_k  = \frac{\mu}{4L} \sum_{k=0}^{K-1} \frac{1}{K \sqrt{k}} =  \mathcal{O}(1/\sqrt{K})$, which can be within $(0,1)$ as $K$ is large enough. Thus the first term of right-hand side of \cref{eqn:thm1} will trend to be zero at linear rate as $T$ grows. For the rest terms of right-hand side of \cref{eqn:thm1}, we notice that $\sum_{t=0}^{T-1} \left(1-\beta\right)^t \le \frac{1}{\beta} = \mathcal{O}(\sqrt{K})$, $\sum_{k=0}^{K-1}\eta_k^2 = \mathcal{O}(\frac{\ln K}{K^2})$,  $\sum_{k=0}^{K-1} \frac{k}{K} \sum_{j=0}^{k-1} \eta_j^2 = \mathcal{O}(\frac{\ln K}{K})$, $\sum_{k=0}^{K-1} \eta_k(1 - 2LK \eta_k)(LK\eta_k -1)<0$.
However, we are still unsure about the sign of $\Delta_{m(t)}$ by inspecting its definition.
By assuming $|\Delta_{m(t)} | \leq \Delta_{\max}$, $|\theta_{m(t)} | \leq \theta_{\max}$ and $\tau_{m(t)} \leq \tau_{\max}$, the asymptotic expression of  the right-hand side of \cref{eqn:thm1} can be written as $O\left(\left(1- \beta\right)^T + \frac{\ln K}{\sqrt{K}}\left(\theta_{\max}^2+G^2+\tau_{\max}\right)\right) + \mu \Delta_{\max}$, which will converge to $\mu \Delta_{\max}$ at linear rate with $T$ and at rate $\mathcal{O}(\ln K/\sqrt{K})$ with $K$.
\item To further speedup convergence rate with $K$, we explore another way of $\{\eta_k\}$ configuration, which sets $\eta_k = \frac{1}{2L K^q}$ with $q\geq 2$. This set of $\{\eta_k\}$ also fulfills $\eta_k \leq \frac{1}{2LK}$ for $k\in \mathcal{K}$. By following the similar analytical procedure as for previous $\{\eta_k\}$ configuration, the optimality gap can be bounded as $\mathcal{O}((1-\beta)^T + \frac{1}{K^{q-1}})+ \mu \Delta_{\max}$, which achieves a faster convergence rate with $K$. This convergence rate is also adjustable by $q$. 
\item Note that the optimality gap is not zero because of the existence of $\Delta_{m(t)}$. In this setup, we partly relax the assumption of data heterogeneity. Specifically, we still respect the data heterogeneity among clients within any cluster but assume the data distribution among clusters to be identical. This assumption is still reasonable. To given an instance, for the integrated LEO satellite terrestrial network demonstrated in Section \ref{apx:app}, every bypassing LEO satellite is an ES node but covers the same group of clients on the ground. Hence the data distribution over clusters is identical and the expectation of $\Delta_{m(t)}$ can be regarded as 0. In this case, with previously suggested setup of $\{\eta_k\}$, zero optimality gap is achieved.
\end{itemize}	
\end{remark}

\subsection{For non-convex loss function}
\begin{theorem}
	\label{thm:nonconvex}
	With the holding of \cref{ass:smooth} and \cref{ass:heterogeneity}, {for the set of $\eta_k$ such that $\eta_k \leq {1}/{(LK)}, \forall k\in \mathcal{K}$}, there is
	\begin{align}
		\label{eqn:thm2}	
 \mathbb{E}\left[\frac{1}{T}\sum_{t=0}^{T-1} \left\|\nabla F\left(w^t\right)\right\|^2\right] 
  \le\frac{4\left[F\left(w^0\right)-F\left(w^*\right)\right]}{T\sum_{k=0}^{K-1} \eta_k} +\left(\frac{2LK\sum_{k=0}^{K-1} {\eta^2_k}}{\sum_{k=0}^{K-1} \eta_k} + 4\right)\frac{1}{T}\sum_{t=0}^{T-1}{ \theta_{m(t)}^2} +\frac{2}{T}\sum_{t=0}^{T-1} \sigma^2 
	\end{align}
\end{theorem}

 \begin{proof}
 	See \cref{apx:nonconvex}.
 \end{proof}

\begin{remark} \label{r:nonconvex}
The bound in \cref{eqn:thm2} clearly reveals how the sampling variance in one cluster, which is indicated by $\theta_m^2$ for $m\in \mathcal{M}$, and the data heterogeneity among clusters, which is indicated by $\sigma^2$, affect the convergence performance. 
To achieve a tighter bound, we made the following configurations of $\{\eta_k\}$:
\begin{itemize}[leftmargin=*]
\item Like the discussion in \cref{r:convex}, we first set { $\eta_k={1}/{(LK\sqrt{k+1})}$ for $k\in \mathcal{K}$.} In this case, ${1}/{\sum_{k=0}^{K-1} \eta_k} = \mathcal{O}(\sqrt{K})$ and ${(K \sum_{k=0}^{K-1} \eta_k^2)}/{\sum_{k=0}^{K-1} \eta_k} = \mathcal{O}\left({\ln K}/{\sqrt{K}}\right)$, then the right-hand side of \cref{eqn:thm2} is upper bounded by 
\begin{equation}
O\left(\frac{\sqrt{K}}{T}\left(F(w^0)-F(w^*)\right)+\frac{\ln K}{\sqrt{K}} \sum_{t=0}^{T-1} \theta_{m(t)}^2\right) + \frac{4}{T} \sum_{t=0}^{T-1} \theta_{m(t)}^2 + 2 \sigma^2, \label{e:nonconvex_gap_bound}
\end{equation}
which will converge if $T$  grows faster than $\sqrt{K}$ and  $\mathcal{O}\left(({\sqrt{K}}/{T})\left(F(w^0)-F(w^*)\right)+ ({\ln K}/{\sqrt{K}}) \sum_{t=0}^{T-1} \theta_{m(t)}^2\right)$ goes to zero with the increase of $K$.  
\item To further speedup convergence and express convergence gap in a simple way, we alternatively { set $T^{q_1}= K$ and  $\eta={1}/({LT^{q_2}})$ with $q_1 \in \left(0,1\right)$, $q_2 \ge q_1$, and $1 + q_1 > q_2$. } With such a configuration, the right-hand side of \cref{eqn:thm2} is bounded by 
\begin{gather}
\mathcal{O}\left(\frac{1}{T^{1+q_1 - q_2}}\left(F(w^0)-F(w^*)\right)+\frac{1}{T^{q_2 - q_1}} \sum_{t=0}^{T-1} \theta_{m(t)}^2\right) + \frac{4}{T} \sum_{t=0}^{T-1} \theta_{m(t)}^2 + 2 \sigma^2,
\label{e:nonconvex_gap_bound_2}
\end{gather}
which converges at a rate $O(1/T^{1+q_1-q_2})$.
The convergence rate can be easily adjusted by changing $q_1$ or $q_2$, such that $q_1 \in \left(0,1\right)$, $q_2 \ge q_1$, and $1 + q_1 > q_2$, and will be surely faster then the bound in \cref{e:nonconvex_gap_bound} as $T$ and $K$ grows.
\item From \cref{e:nonconvex_gap_bound} and \cref{e:nonconvex_gap_bound_2}, it can be also observed that these bounds will converge to zero as the variance lead by sampling in one cluster (described by $\theta_m^2$ for $m\in \mathcal{M}$) and the data heterogeneity among the clusters (represented by $\sigma^2$) vanish. Hence stationary point is reached.
\end{itemize}		
\end{remark}
\begin{table*}[tb]
\caption{The global model performance of different algorithms on MNIST, CIFAR-10 and CIFAR-100 with $\lambda$=0.3 or $\lambda$=0.6. We set 100 clients, 10 ESs, T=4000, K=20 and other general settings.}
\label{tab:sota}
\renewcommand{\arraystretch}{1.5}
\resizebox{\textwidth}{!}{%
\begin{tabular}{cccccccccc}
\hline
Dataset                    & Loss Function & \multicolumn{2}{c}{Fed-CHS}       & \multicolumn{2}{c}{FEDAVG}      & \multicolumn{2}{c}{WRWGD}          & \multicolumn{2}{c}{Hier-Local-QSGD}    \\ \hline
                           &               & Dirichlet(0.3) & Dirichlet(0.6) & Dirichlet(0.3) & Dirichlet(0.6) & Dirichlet(0.3) & Dirichlet(0.6) & Dirichlet(0.3) & Dirichlet(0.6) \\ \hline
\multirow{2}{*}{MNIST}     & MLP            & 0.9811         & 0.9811         & \textbf{0.9813}         & \textbf{0.9821}         & 0.9761        & 0.9767       & 0.9705         & 0.9751         \\ \cline{2-10} 
                           & LENET         & \textbf{0.9921}         & \textbf{0.9918}         & 0.9912         & 0.9916         & 0.9891        & 0.9888         & 0.9909         & 0.9907         \\ \hline
\multirow{2}{*}{CIFAR-10}  & MLP            & \textbf{0.6249}         & 0.6324         & 0.5775         & \textbf{0.6380}         & 0.5166         & 0.5338         & 0.5685         & 0.5772         \\ \cline{2-10} 
                           & LENET         & \textbf{0.8198}         & \textbf{0.8237}         & 0.7364         & 0.8042         & 0.7378         & 0.7631         & 0.6805         & 0.7545         \\ \hline
\multirow{2}{*}{CIFAR-100} & MLP            & \textbf{0.3316}         & \textbf{0.3322}         & 0.3067         & 0.3123         & 0.1950         & 0.2273         & 0.2249         & 0.2636         \\ \cline{2-10} 
                           & LENET         & \textbf{0.4766}         & \textbf{0.4790}         & 0.3151         & 0.3184         & 0.2920         & 0.3245         & 0.3558         & 0.4021         \\ \hline
\end{tabular}%
}
\end{table*}

\section{Experiments}
\label{s:exp}
Several important experimental results are presented in this section, whose detailed discussions are postponed to \cref{app:hyperparameter} and \cref{app:futher_comparison}.

\subsection{Setup}


{\bf Datasets and Models:}
{We experiment with three datasets, MNIST, CIFAR-10 and CIFAR-100 using the LENET \citep{726791,lecun2015lenet} and Multi-Layer Perceptron (MLP) \citep{pmlr-v162-yue22a}  models. }
The details for datasets are present in \cref{apx:exp}.
The non-IID settings adopted by us involve Dirichlet($\lambda$): the label distribution on each device follows the Dirichlet distribution with $\lambda>0$ being a concentration parameter.
A larger value of $\lambda$ indicates a more homogeneous distribution across each dataset.
The loss function in LENET model and MLP model can be regarded as non-convex and convex, respectively.
More details can be found in \cref{apx:exp}.

{\bf Baselines:}
{
To empirically highlight the training performance of our proposed framework, we first compare it in the general settings with FedAvg \citep{mcmahan2017communication}, Weighted RWGD (WRWGD) \citep{8778390}, Hier-Local-QSGD \citep{9834296} based on MNIST, CIFAR-10 and CIFAR-100 datasets. 
{
Then, in order to evaluate the communication overhead, we compare our proposed framework in the similar settings with FedAvg compressed by QSGD \citep{NIPS2017_6c340f25} and Hier-Local-QSGD based on MNIST, CIFAR-10 datasets for LENET model. 
}
}

{\bf Metrics:}
To investigate the performance of Fed-CHS, we compare it with other algorithms by measuring the accuracy in the test. 
A higher test accuracy means that the performance of the associated algorithm is better. 
During the training process, we also evaluate the communication overhead for passing model parameters.
{
Specifically, we compare the total communication cost for passing model parameters under different algorithms to reach specific threshold of test accuracy, denoted as $\Gamma$.
}
{
For MNIST, CIFAR-10 and CIFAR-100 dataset, $\Gamma$ is selected to be $0.98, 0.72, 0.40$, respectively, by referring to achievable test accuracy in \cref{tab:sota}.
}

\subsection{Results for Training Performance}
\label{s:acc for baseline}

The experimental results are reported in \cref{tab:sota}.
These results show that Fed-CHS can achieve higher accuracy under the majority of configurations, although it may not be consistently superior to other algorithms.
Particularly noteworthy is our proposed Fed-CHS expands its advantage over other baseline algorithms when the data distribution becomes more heterogeneous.

For the easy datasets (MNIST), Fed-CHS is not always the best but only has a tiny performance gap compared with FedAvg, which achieves highest accuracy in this case.
It is worthy to note that MNISt dataset has limited number of target classes, which can drive every testing algorithm to perform well easily. Hence some slight disadvantage on MNIST for Fed-CHS does not mean it is inferior to other baselines. 
{
For the MLP model on the CIFAR-10 or CIFAR-100 datasets, each algorithm performs poorly, especially for WRWGD and Hier-Local-QSGD. This explains the relatively low level of test accuracy for all the testing algorithms.
The above situation alleviates a lot for LeNet model, except of FedAvg on CIFAR-100 dataset.
}

Additionally, to highlight the stability of our algorithm, the performance of Fed-CHS and other algorithms in different settings where clients use different models and Dirichlet parameters under the datasets of CIFAR10, CIFAR100, and MNIST is plotted with round number $T$, as given in \cref{app:futher_comparison}.
An improvement of accuracy over baselines ranges from 
$0.02\%$ and $33.95\%$ can be observed. 
Further analysis can be also found from \cref{app:futher_comparison}.

\subsection{Results for Communication Overhead}
\label{s:exp for com cost}

\begin{figure}[ht]
    \centering
    \subfigure[MNIST]{
        \includegraphics[width=0.3\linewidth]{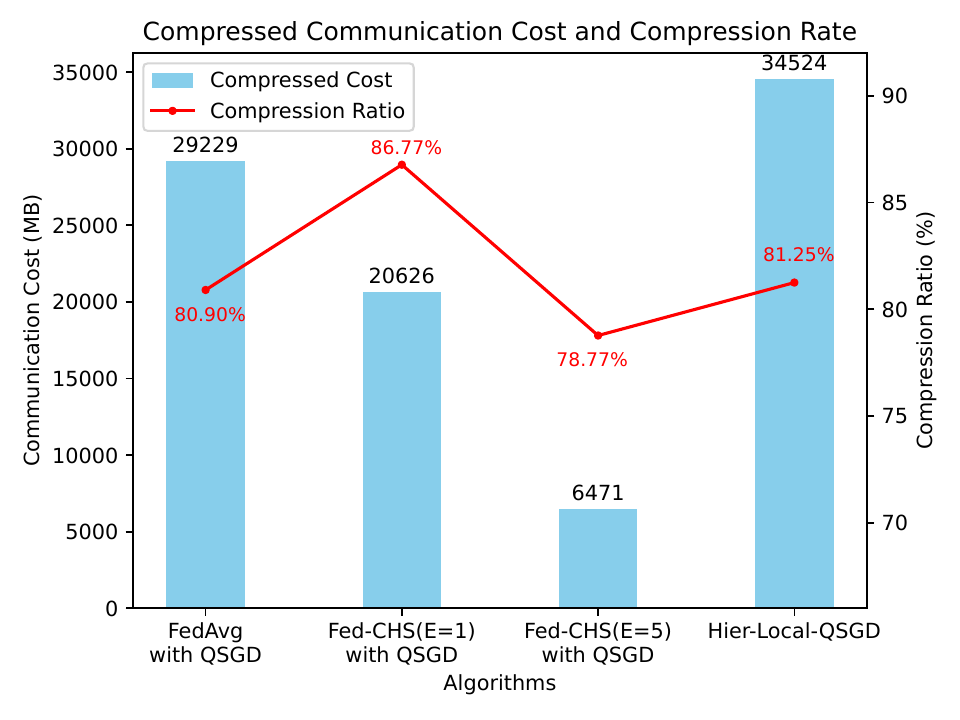}
    }%
    \subfigure[CIFAR-10]{
        \includegraphics[width=0.3\linewidth]{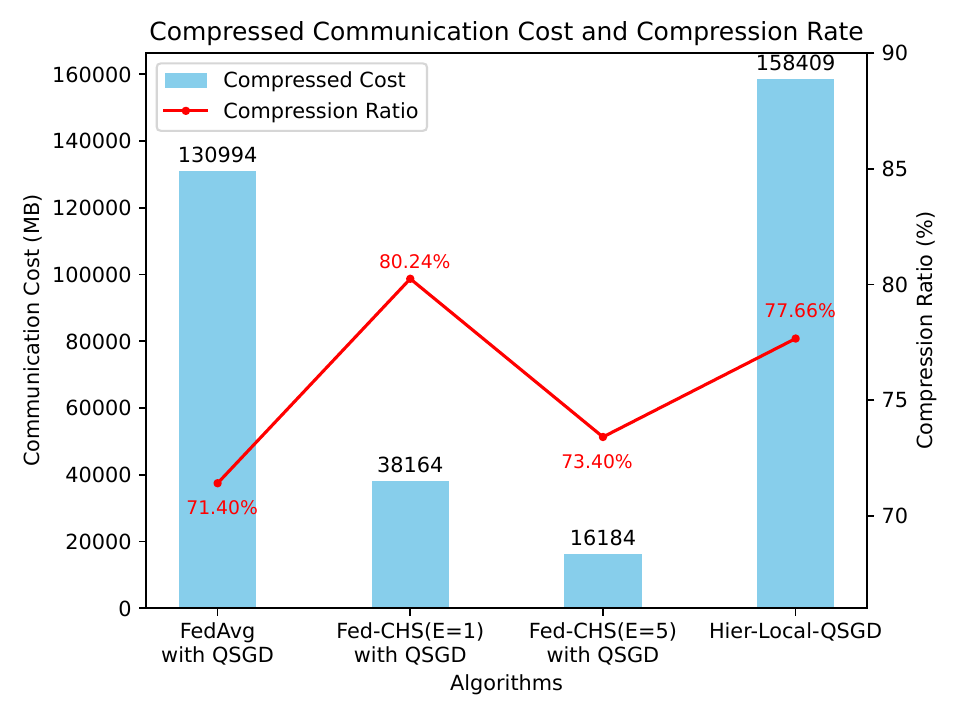}
    }%
    \subfigure[CIFAR-100]{
        \includegraphics[width=0.3\linewidth]{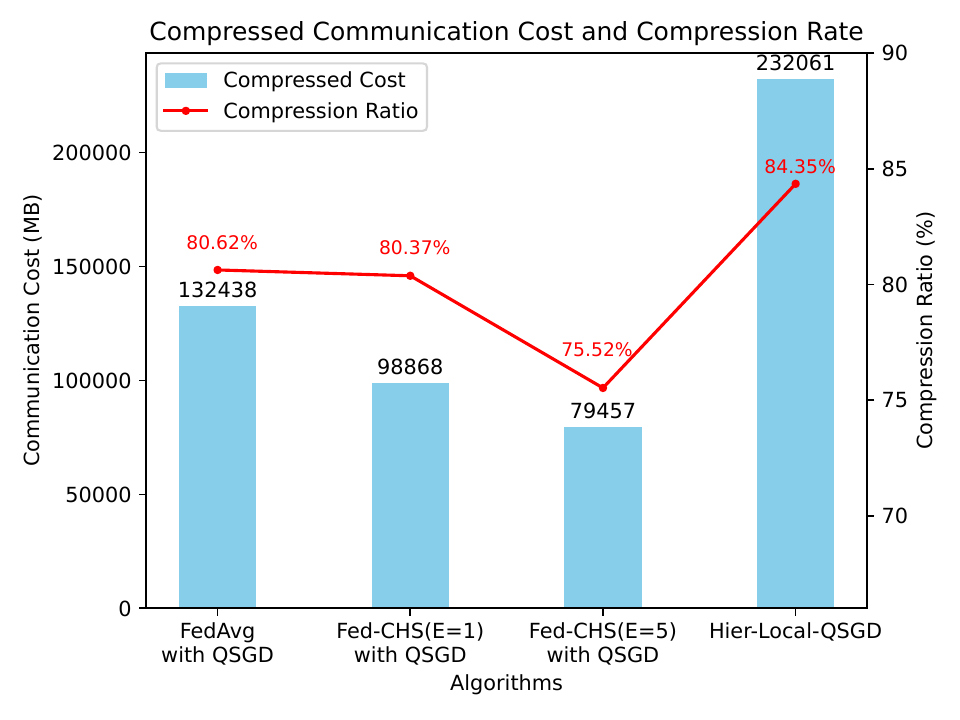}
    }%
    \caption{{Results of communication overhead for different algorithms and datasets.}}
    \label{fig:communication overhead}
\end{figure}

{
We recorded the total communication overhead (in unit of information bits) of our algorithm and the baselines to achieve default $\Gamma$ with compression, and the percentage reduction in communication overhead with the aid of compression, in \cref{fig:communication overhead}. 
Therefore the communication overhead without compression can be also inferred from  \cref{fig:communication overhead}.
}
Notably, for Fed-CHS, a total number of aggregations within one cluster is 20. 
Hence every client interacts with associated ES by 20 times when $E=1$, which will turn to be 4 times when $E=5$.
{
It is evident that for datasets MNIST, CIFAR-10, or CIFAR-100, our algorithm significantly outperforms other baseline algorithms in terms of total communication overhead with or without compression.
}
This result is reasonable because the passing of aggregated model parameter among ESs, rather than further aggregating them at a PS, can help to save a lot of communication overhead compared with FedAvg (with QSGD) or Hier-Local-QSGD, and does not degrade the convergence performance too much.
There is one more thing to highlight, the counted communication for Fed-CHS happens between a client and an associated ES or an ES and its neighbor ES, all of which only involve one-hop communication, while the counted communication between a client and the PS in FedAvg (with QSGD) may be over a long distance that will ask for multiple hops of routing and incur heavier communication burden. 

\section{Conclusion}
\label{s:conclusion}

In this paper, we propose a new federated learning algorithm Fed-CHS, which introduces SFL into HFL for the first time to further reduce communication overhead for passing model parameters. Theoretical analysis guarantees the convergence performance of the Fed-CHS for strongly convex or non-convex loss function under various setup of data heterogeneity.
Extensive experiments are conducted to verify the advantage of our proposed Fed-CHS in terms of test accuracy, communication overhead saving, and convergence over existing methods.
{
In the future, we will extend our theoretical analysis from the setup of single-step local iteration to multi-step, which is more general. And the potential of this setup has been verified in existing experiments on communication overhead.
}

\newpage

\bibliography{Reference}
\bibliographystyle{apalike}

\newpage
\appendix
\onecolumn

\tableofcontents

\newpage



\section{Implementation Details}
\label{apx:exp}
We experiment with three datasets, MNIST, CIFAR-10 and CIFAR-100 using the LENET \citep{726791,lecun2015lenet} and Multi-Layer Perceptron (MLP) \citep{pmlr-v162-yue22a}  models.
Firstly, we describe the datasets as below: 
\begin{itemize}[leftmargin=*]
	\item \textbf{MNIST:} It contains 60000 train samples and 10000 test samples. 
	Each sample is a $28 \times 28$ scale image associated with a label from 10 classes. 
	\item \textbf{CIFAR-10:} It is a dataset that includes 10 different types of color images, with 6000 images in each category, totaling 60000 images. 
	All images are divided into 50000 train samples and 10000 test samples. 
	Each image has a resolution of $32 \times 32$ pixels and includes 3 color channels (RGB). 
	\item \textbf{CIFAR-100:} It includes 50000 train samples and 10000 test samples, each associated with 100 possible labels. 
	The sample in CIFAR-100 is a $32 \times 32$ color image.
\end{itemize}
The non-IID settings we adopt involve Dirichlet($\lambda$): the label distribution on each device follows the Dirichlet distribution, where $\lambda$ is a concentration parameter ($\lambda > 0$).
A larger value of $\lambda$ indicates a more homogeneous distribution across each dataset.
The loss function in LENET model is non-convex and the loss function in MLP model is convex.
Next, we represent the network model used in experiments as below:
\begin{itemize}[leftmargin=*]
	\item \textbf{MLP}: There are two fully connected layers in MLP, apart from the input layer and the output layer. 
    For MNIST, these two fully connected layers both have 200 neurons.
	For CIFAR-10 and CIFAR-100, the amount of neurons in the two fully connected layers is 256 and 512, respectively.
	We all set ReLU as the activation function of the two middle layers.
	\item \textbf{LENET}: Different from the MLP, we add two sets of concatenated convolution layer and pooling layer between the input layer and the fully connected layer in LENET. 
	{For these two convolution layers, we utilize 64 and 256 convolution kernels to train MNIST respectively, each with a size of $5 \times 5$.
    And we utilize two 64 convolution kernels having been adopted for MNIST to train CIFAR-10 and CIFAR-100 respectively. }
	Additionally, all the pooling layers are set to be 2x2. 
	In case of MNIST dataset, there are two fully connected layers with dimensions $512 \times 512$ and $128 \times 128$, respectively. 
	For CIFAR-10 and CIFAR-100 datasets, the two fully connected layers have dimensions of $384 \times 384$ and $192 \times 192$.
\end{itemize}

To carry out our experiments, we set up a machine learning environments in PyTorch 2.2.2 on Ubuntu 20.04, powered by four RTX A6000 GPUs and AMD 7702 CPU. 

\section{Issues about Hyper-parameters}
\label{app:hyperparameter}
\subsection{Hyper-parameter Settings}
In this work, we generally employed 100 devices to simulate a substantial number of participating clients in practical FL. 
For the hierarchical architecture, we set 10 ES's with each ES being assigned with some clients. 
All clients are selected to generate the global model. 
For non-IID settings of the datasets, the parameter $\lambda$ of Dirichlet needs to be set, we use $\lambda=0.6$ generally and we also use the setting group with $\lambda=0.3$ and $\lambda=0.6$.

Furthermore, we ran these algorithms in $T=4000$ rounds.
{
In Fed-CHS and Hier-Local-QSGD, the local communication round number within ESs is $K=T^{q_1}=20$ and the value of the learning rate is set as $\frac{1}{K\sqrt{k+1}}$ in the $k$-th communication. 
And for FedAvg and Weighted RWGD, we set the training epochs in clients to be $K=20$ and the value of the learning rate as $\frac{1}{K\sqrt{k+1}}=\frac{1}{LT^{q_2}}$ in the $k$-th communication.
}  
For Fed-CHS and Weighted RWGD, we first randomly generate the network topology 
before these algorithms begin to train. 
The ES (or client) node is required to be connected with at most three other ES (or client) nodes for the network topology of Fed-CHS (or Weighted RWGD), which uses a relatively sparse connection approach to better mimic the physical connectivity.
{
To evaluate the total communication cost, we represent each parameter by 32 bits.
For ease of comparison, in Fed-CHS and FedAvg, each client performs one local iteration per round, while in Hier-Local-QSGD, each client performs 5 local iterations per round. Both Fed-CHS and Hier-Local-QSGD conduct 20 iterations per round within the cluster.
}

\subsection{Results for Impact of Hyper-parameters}
In this subsection, we examine the impact of the number of local communication rounds, the degree of data heterogeneity and the number of ES's on our proposed Fed-CHS.
Additionally, we also examine the impact of partial data heterogeneity within one cluster on Fed-CHS.

\cref{fig:k with mlp} and \cref{fig:k with lenet} show that our Fed-CHS's convergence rate varies with the number of local communication rounds. 
In case a smaller $K$ is adopted, which implies a higher learning rate, a faster convergence rate can be achieved.
According to \cref{thm:nonconvex}, the gap of convergence is weakly affected by $K$, the influence of which is even weaker as $T$ grows. This explains the fact that the training accuracy trends to be similar as $T$ increases.
Both \cref{fig:iid with mlp} and \cref{fig:iid with lenet} reveal that severe data heterogeneity may exert influence on the training accuracy significantly.  
But as the dataset distributions among clients tend to be more homogeneous, this kind of influence is likely to decay.
As a comparison, the difference of training accuracy due to data homogeneity is more sensitive to convex loss function, like in MLP, rather than a non-convex loss function, like in LENET. 
\cref{fig:m with mlp} and \cref{fig:m with lenet} clearly illustrate that too many ES's may lead to a bad performance. 
The reason can be explained as follows: With the population of clients unchanged but the number of ES's increased, each ES shares a less fraction of clients, which leads to a less generality of the locally trained model within one cluster.
Sequential updating of this kind of model among clusters become a repetitive correcting process of the model only being adapt to the dataset of the clients in currently active cluster, rather than an evolving process of the model that can represent the global data distribution of all the clients.

\begin{figure*}[ht]
\centering
\subfigure[Different $K$ with MLP]{
\includegraphics[width=0.33\linewidth]{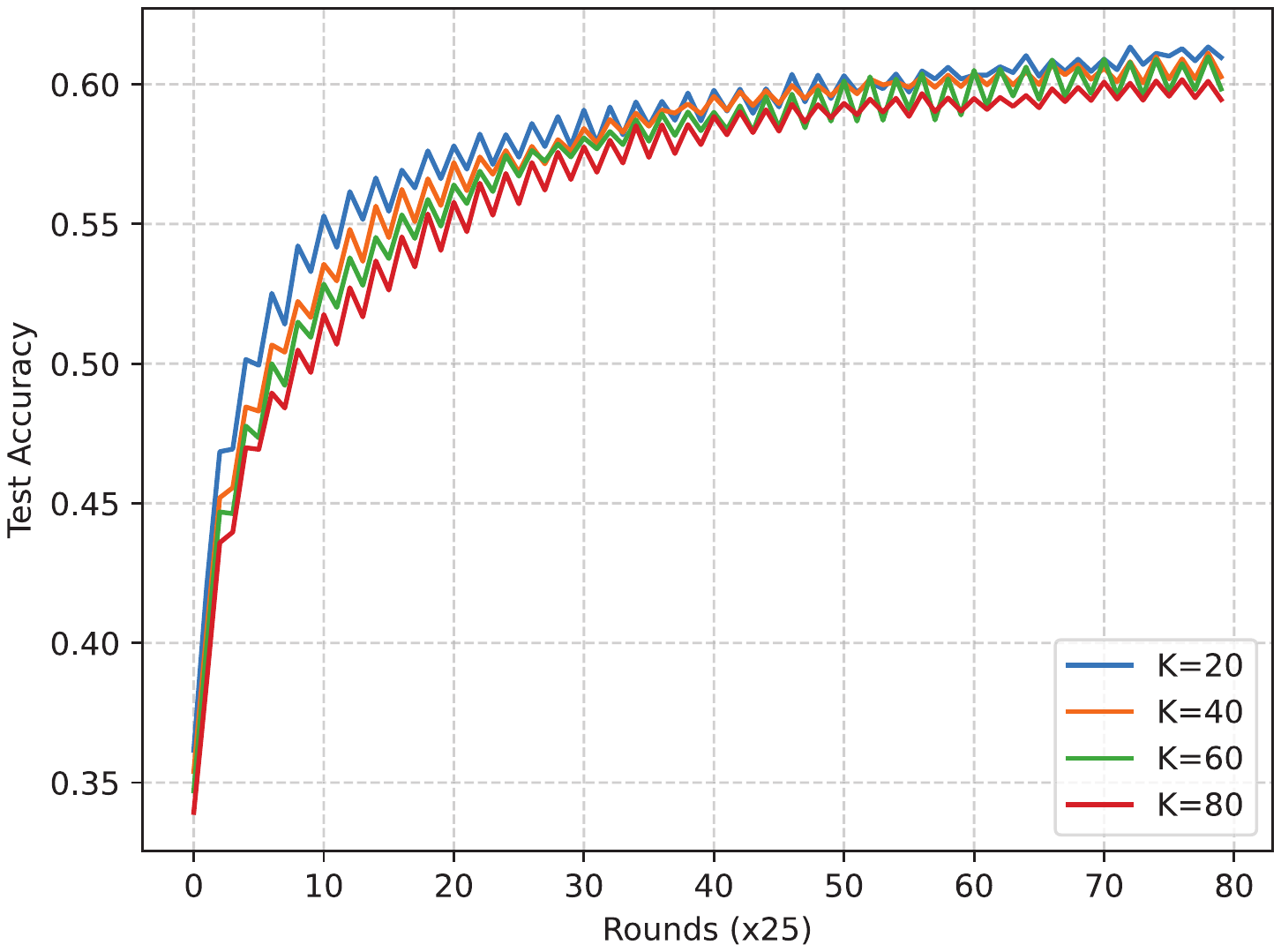}
\label{fig:k with mlp}
}%
\subfigure[Different $\lambda$ with MLP]{
\includegraphics[width=0.33\linewidth]{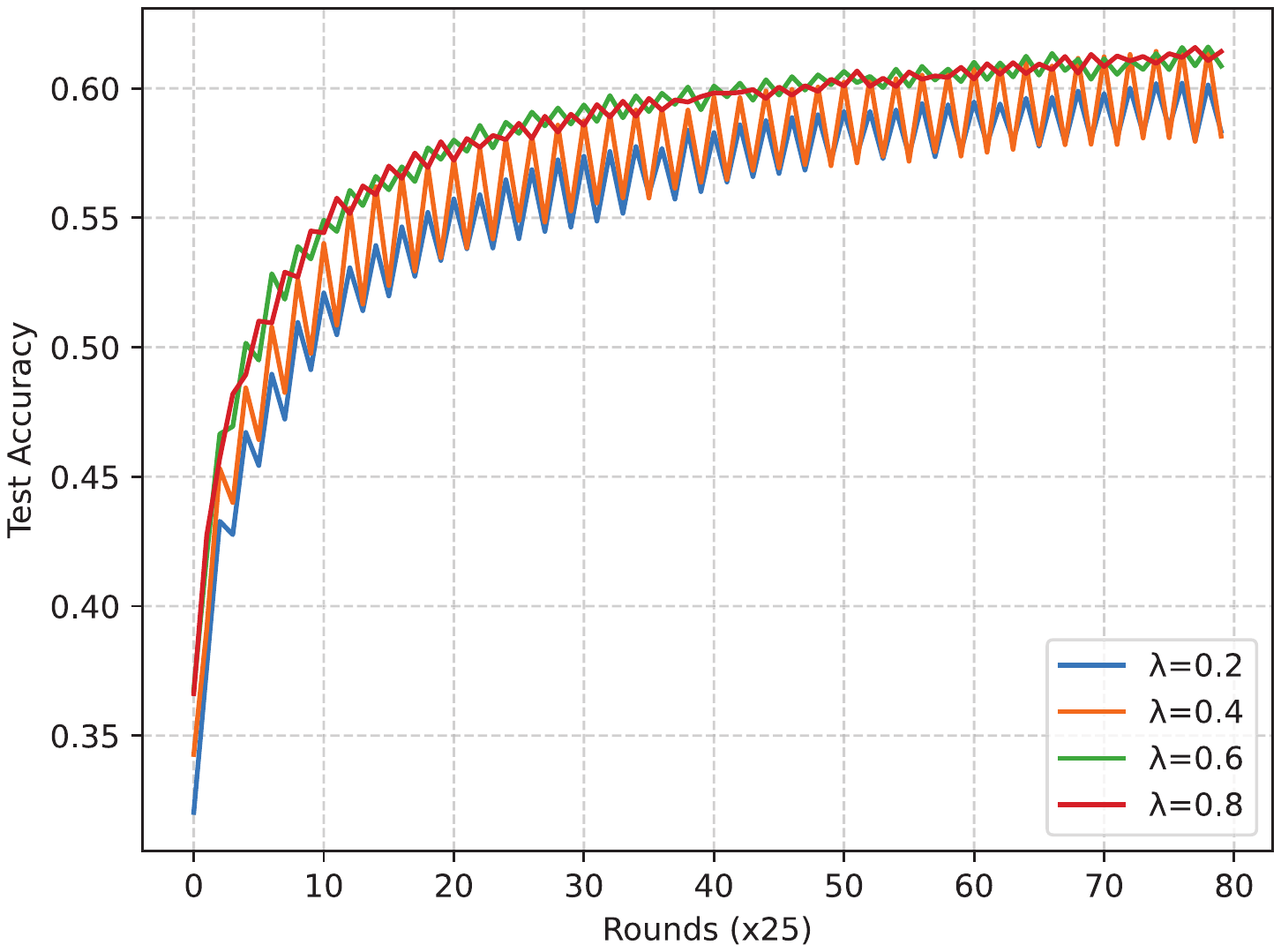}
\label{fig:iid with mlp}
}%
\subfigure[Different $M$ with MLP]{
\includegraphics[width=0.33\linewidth]{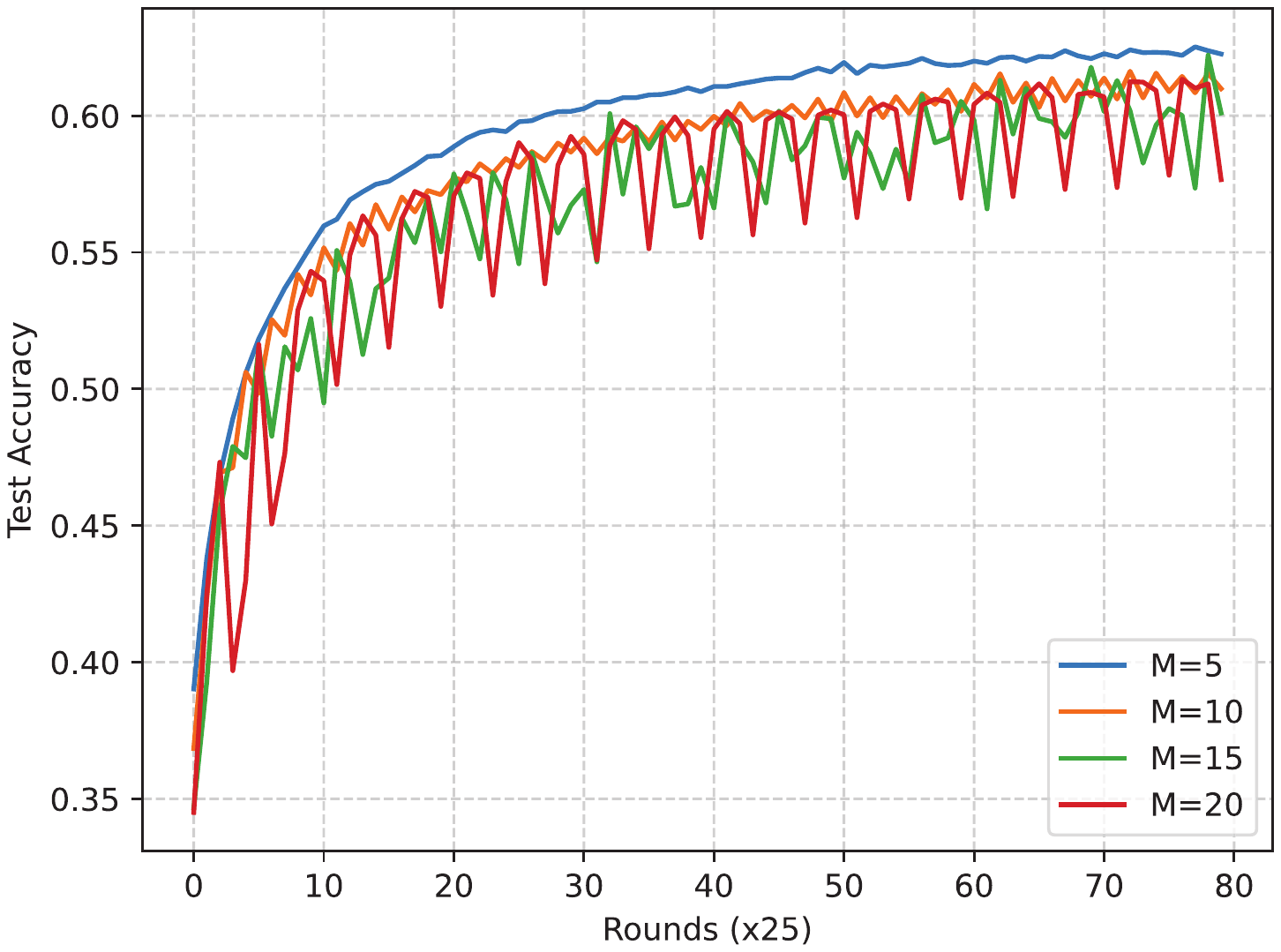}
\label{fig:m with mlp}
}%
\quad            
\subfigure[Different $K$ with LENET]{
\includegraphics[width=0.33\linewidth]{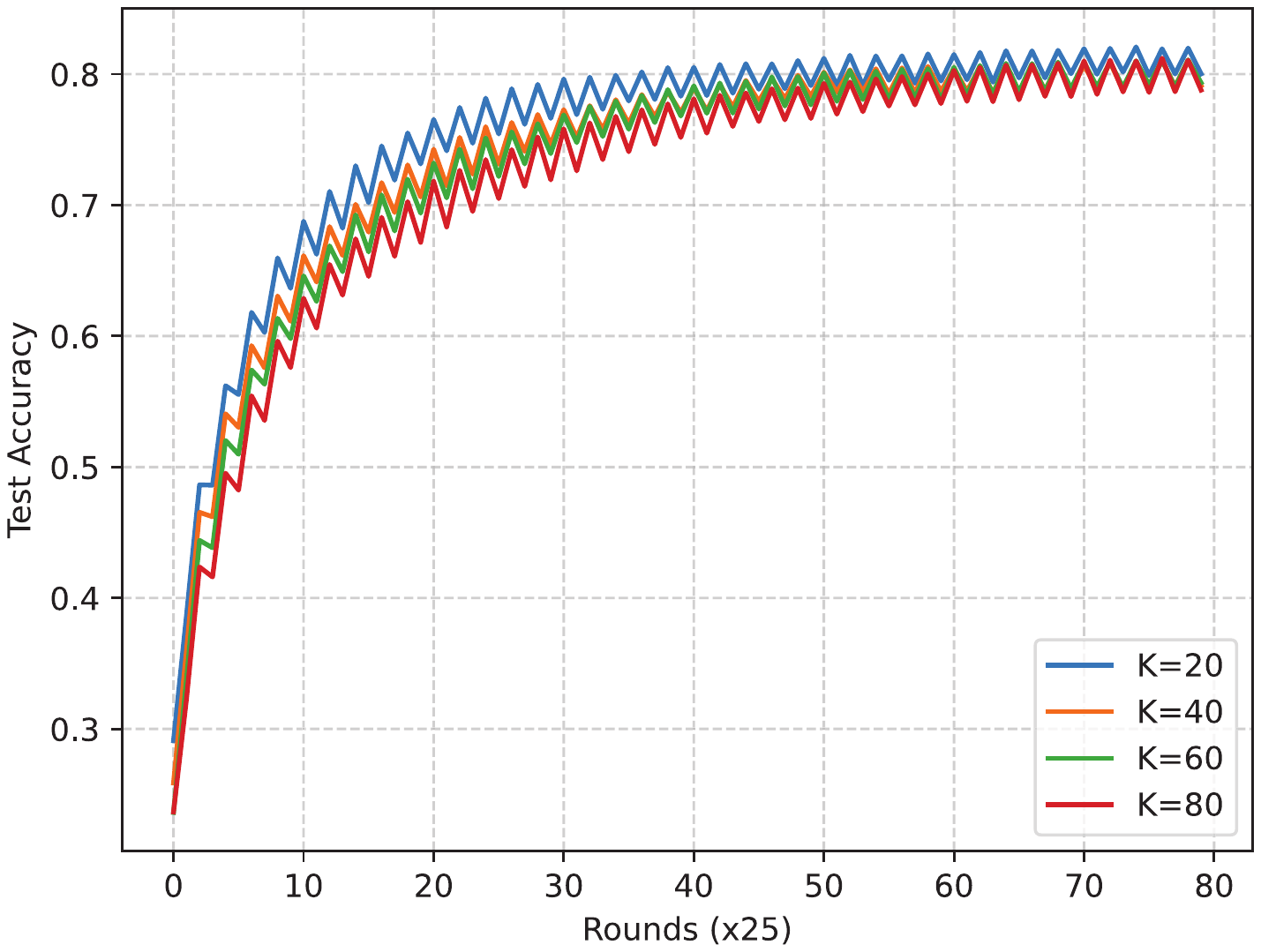}
\label{fig:k with lenet}
}%
\subfigure[Different $\lambda$ with LENET]{
\includegraphics[width=0.33\linewidth]{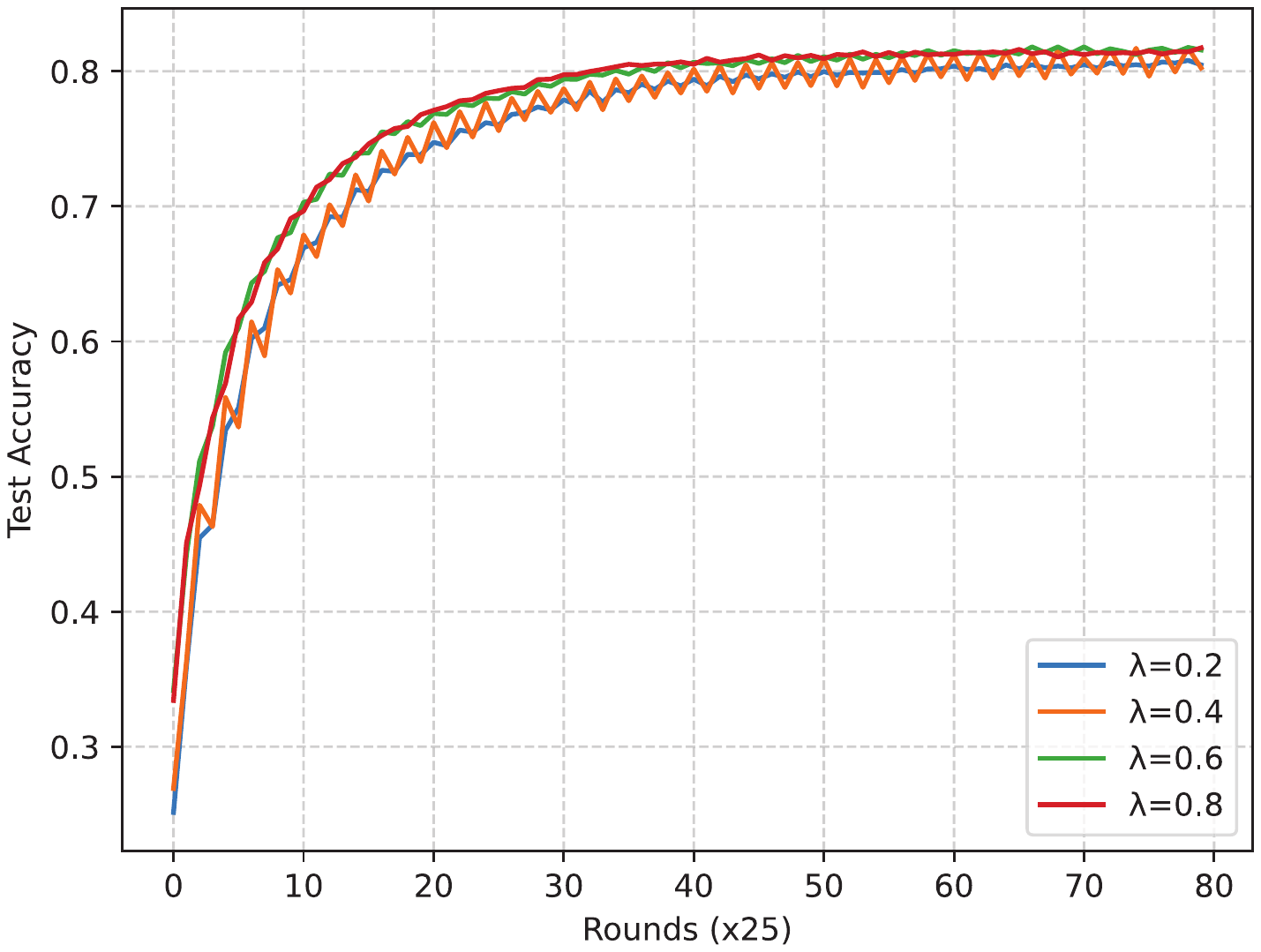}
\label{fig:iid with lenet}
}%
\subfigure[Different $M$ with LENET]{
\includegraphics[width=0.33\linewidth]{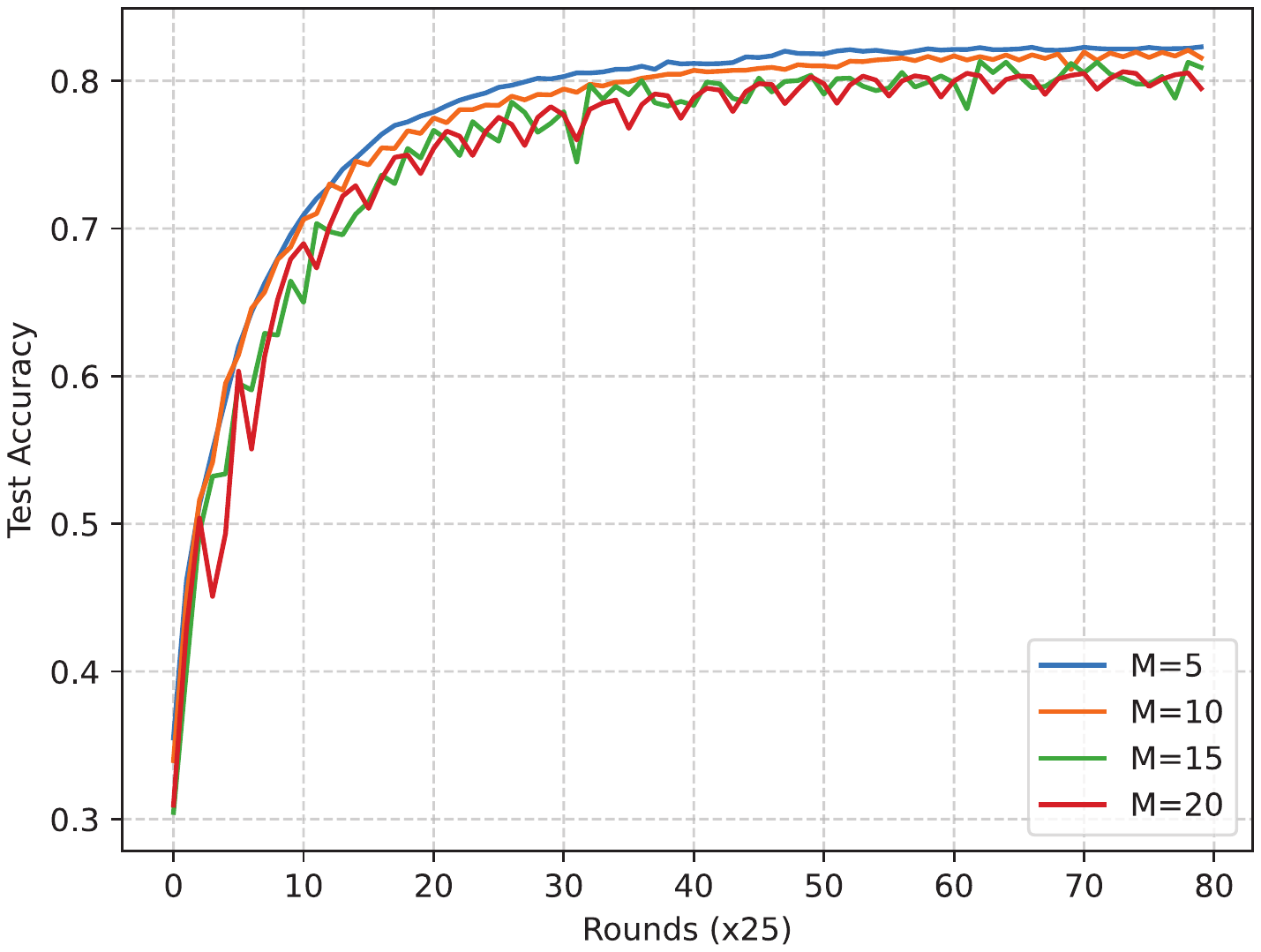}
\label{fig:m with lenet}
}%
\centering
\caption{Results for different hyper-parameters with MLP or LENET}
\label{fig:differnt hp}
\end{figure*}

We also inspect the training performance of Fed-CHS under fully or partially heterogeneous dataset by using LENET model and CIFAR-10 dataset in \cref{fig:esiid}.
When partial data heterogeneity is implemented, the dataset distribution among clusters are IID but the dataset distribution among clients within any cluster is non-IID.
The extent of heterogeneity is still tuned by parameter $\lambda$.
As of training performance, both test accuracy and training loss are probed.
From \cref{fig:esiid}, we observe that the performance gap due to the difference between fully and partially heterogeneous datasets trends to be diminish as $T$ grows.
Also considering the remark in \cref{r:nonconvex}, we can achieve stationary point for non-convex loss function like the LENET model with partially heterogeneous dataset. Then it can be inferred that the performance of Fed-CHS for training a non-convex loss function under fully heterogeneous dataset can be also promised.

\begin{figure}
    \centering
    \subfigure[Test Accuraccy]{
        \includegraphics[width=0.4\linewidth]{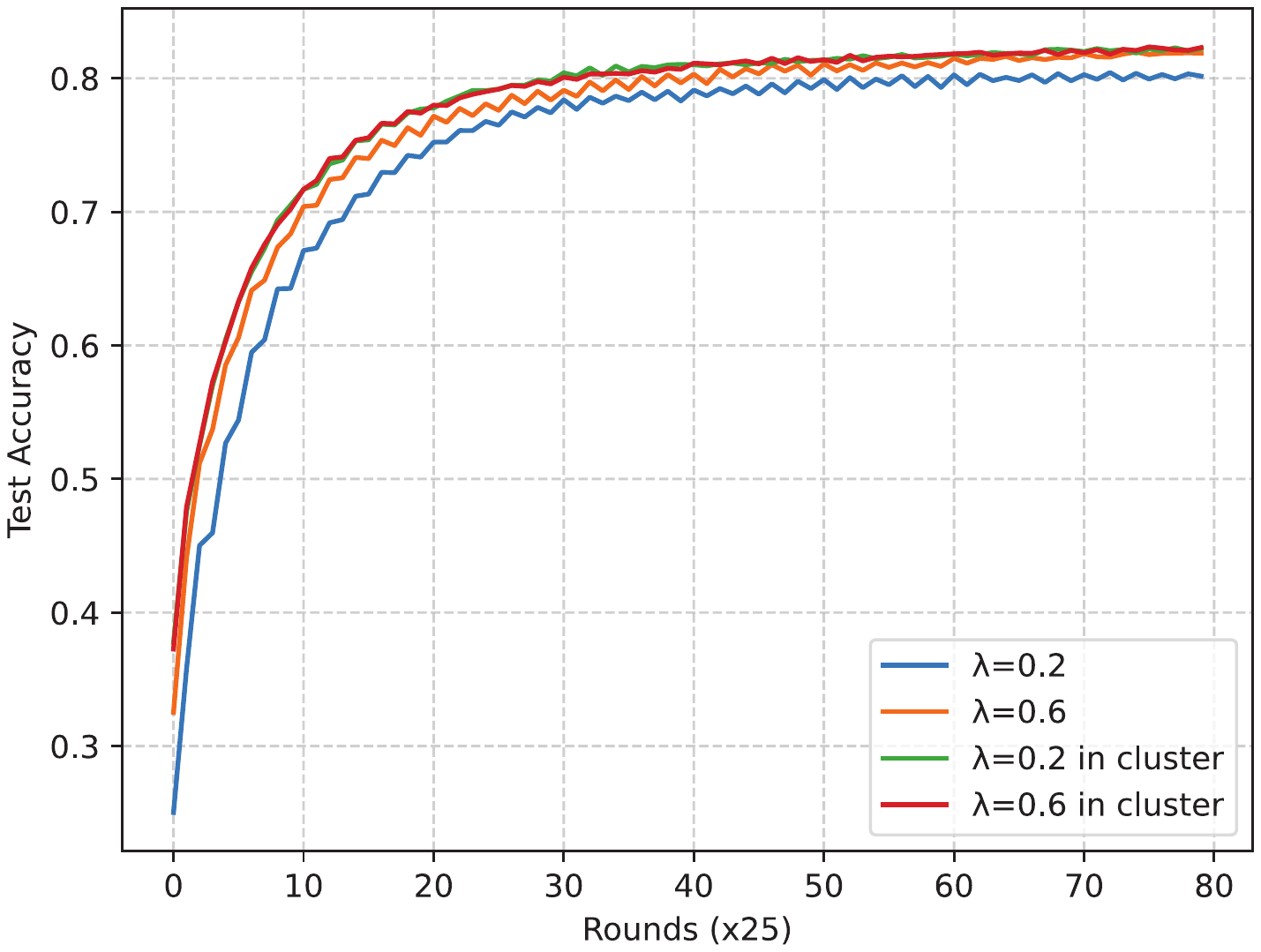}
    }%
    \subfigure[Training Loss]{
        \includegraphics[width=0.4\linewidth]{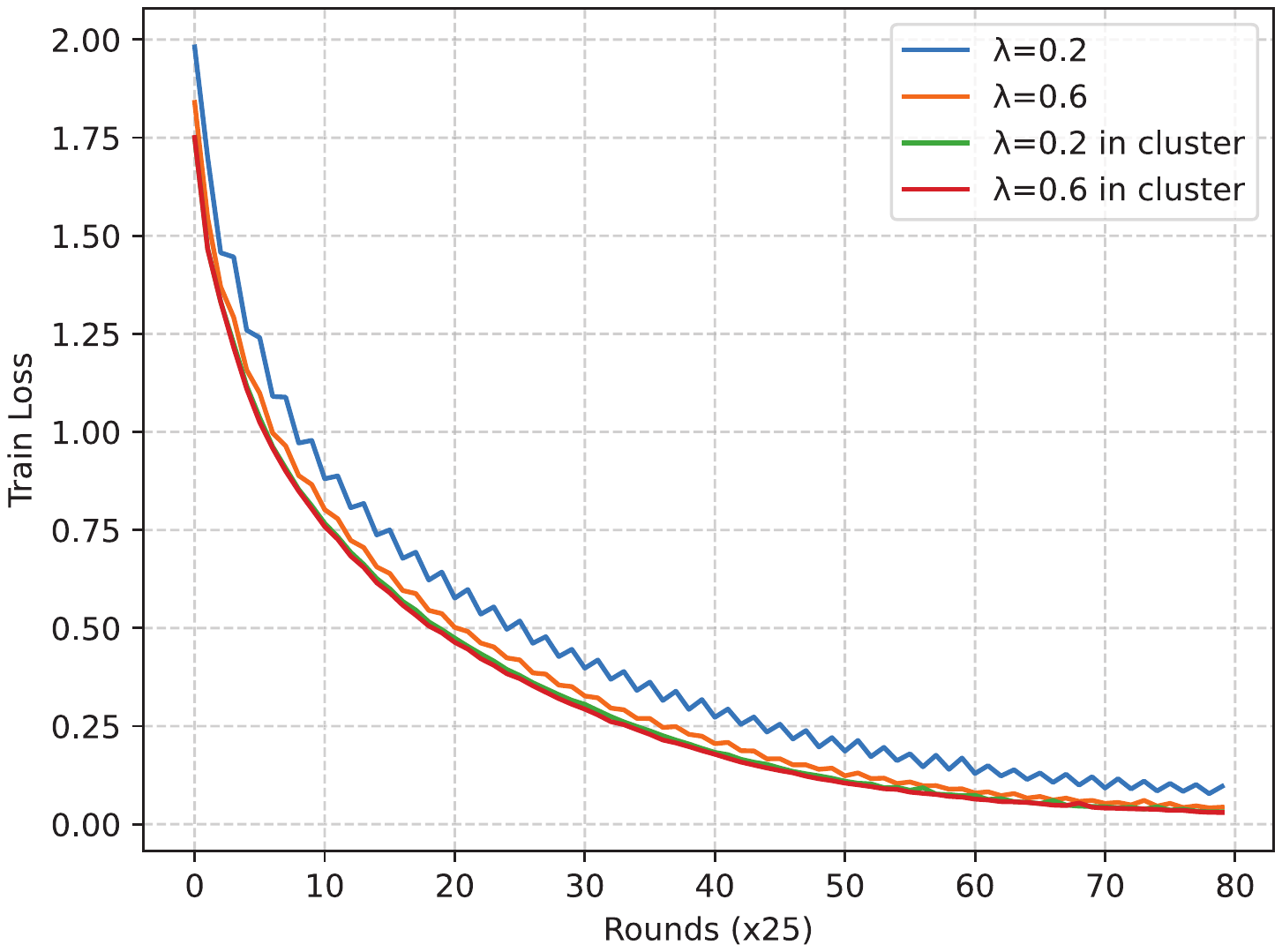}
    }
    \caption{Result for different data heterogeneity in ES}
    \label{fig:esiid}
\end{figure}

\section{Additional Comparison Results for Training Performance}
\label{app:futher_comparison}

\cref{figs:MNIST}, \cref{figs:CIFAR10}, and \cref{figs:CIFAR100} represent the performance of Fed-CHS and other algorithms in different settings where clients use different models and Dirichlet parameters under MNIST, CIFAR-10, and CIFAR-100 dataset, respectively.
For \cref{figs:MNIST}, all algorithms can achieve good performance, but only Fed-CHS has a stable volatility.
In \cref{figs:CIFAR10}, due to the non-IID dataset or the convexity of loss function, which is used to approximate the idea model that is surely non-convex, all the algorithms suffer from different degrees of volatility.
{
Although Fed-CHS doesn't perform the best in convex settings, its performance gap to the optimal one among baseline methods shrinks from {0.1\% to 0.02\%} as data heterogeneity increases.
}
Fed-CHS performs better on volatility and converges faster than all the other algorithms under various settings.
From \cref{figs:CIFAR100}, it can be seen that Fed-CHS still outperforms its counterparts not only on volatility but also on convergence speed under LENET model. 
{
For MLP model, Fed-CHS exhibits slightly less fluctuation and lower convergence speed compared to FedAvg, but its overall model performance surpasses that of all other baseline methods by a significant margin.
}
The above results prove the advantage of our proposed Fed-CHS.

\begin{figure*}[ht]
\centering
\subfigure[MLP and $\lambda=0.3$]{
\includegraphics[width=0.25\linewidth]{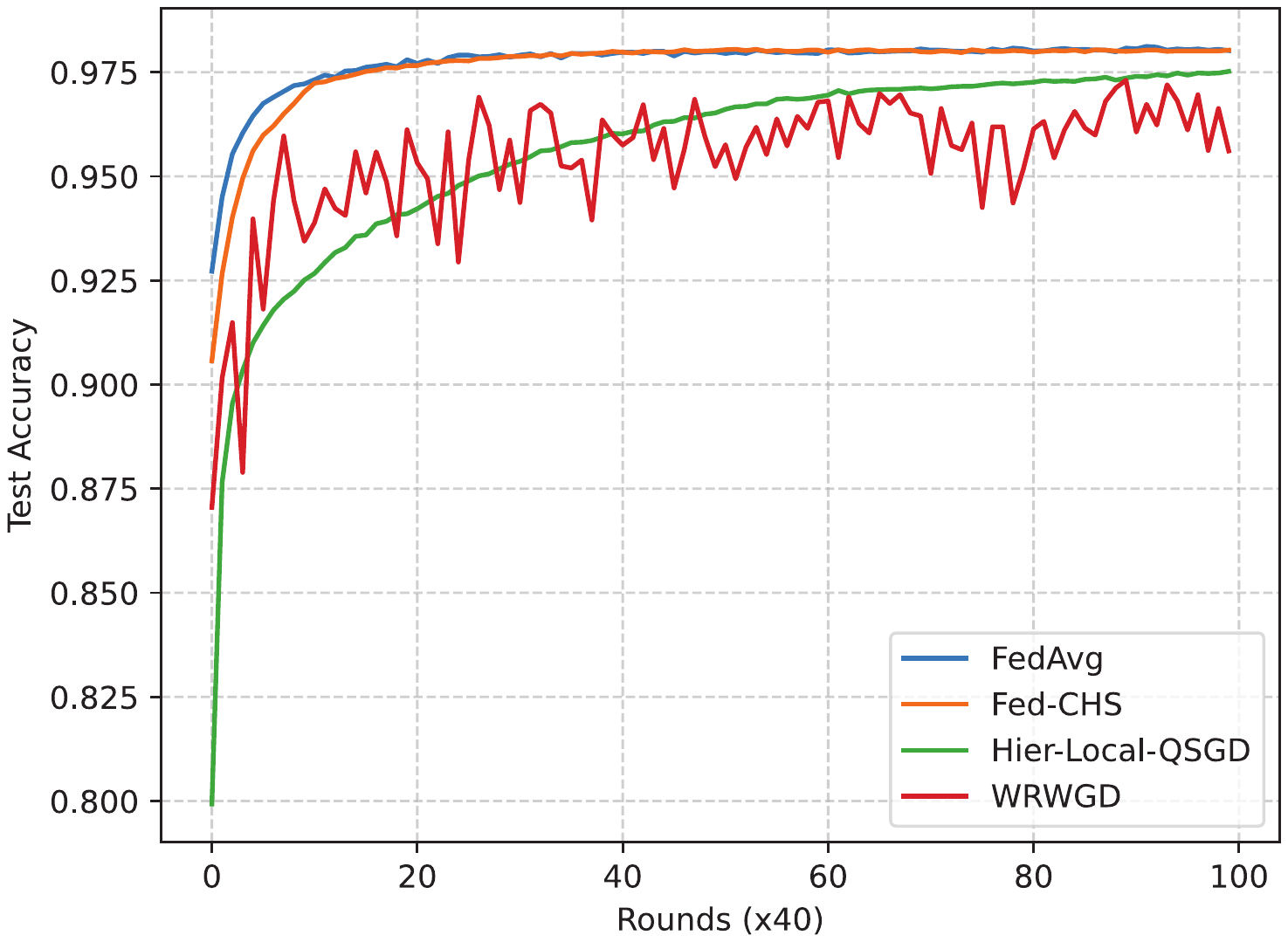}
\label{fig:MNIST-0.3-MLP}
}
\!\!\!\!\!\!\!\!\!\!
\subfigure[MLP and $\lambda=0.6$]{
\includegraphics[width=0.25\linewidth]{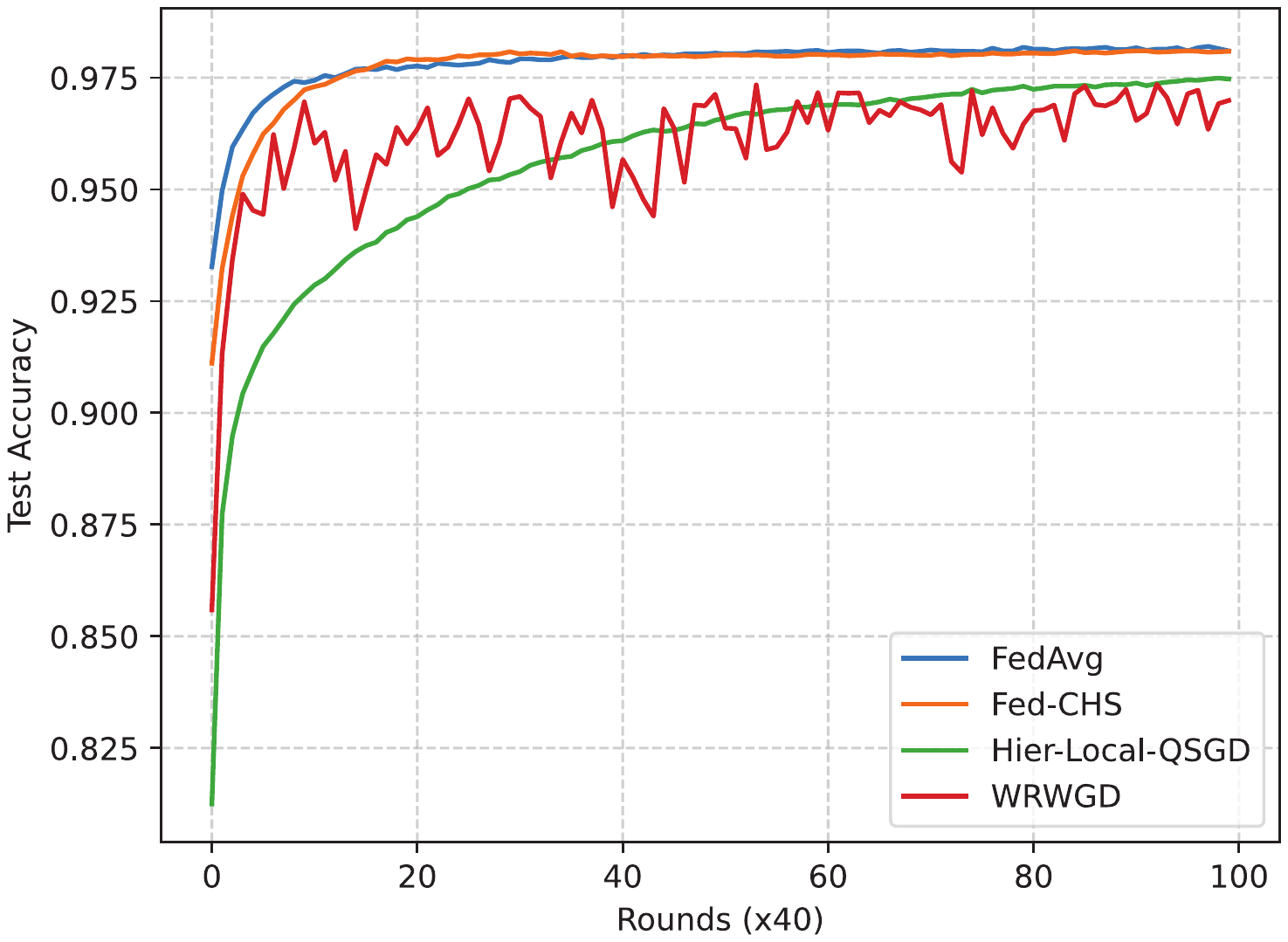}
\label{fig:MNIST-0.6-MLP}
}
\!\!\!\!\!\!\!\!\!\!
\subfigure[LENET and $\lambda=0.3$]{
\includegraphics[width=0.25\linewidth]{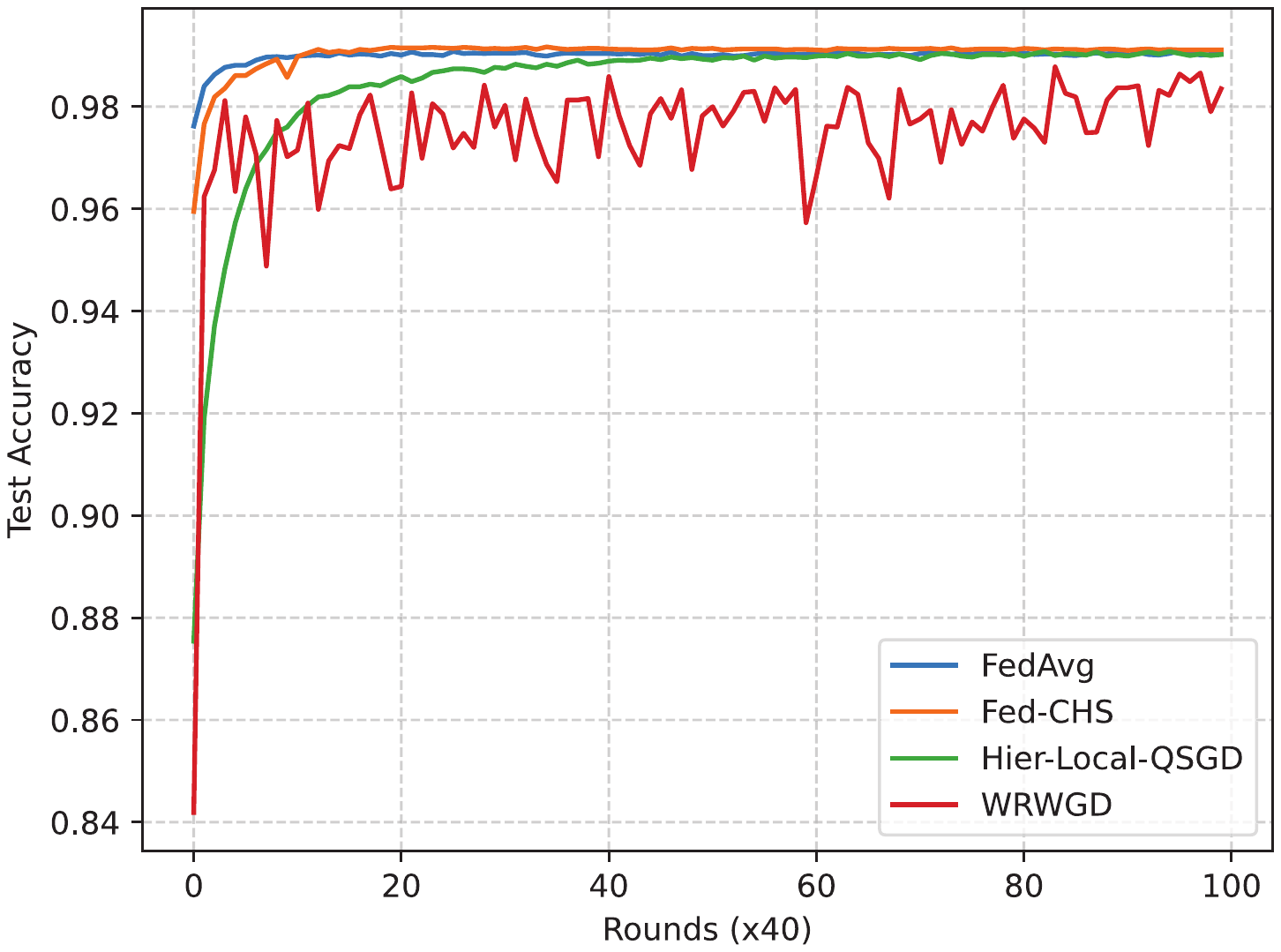}
\label{fig:MNIST-0.3-LENET}
}
\!\!\!\!\!\!\!\!\!\!
\subfigure[LENET and $\lambda=0.6$]{
\includegraphics[width=0.25\linewidth]{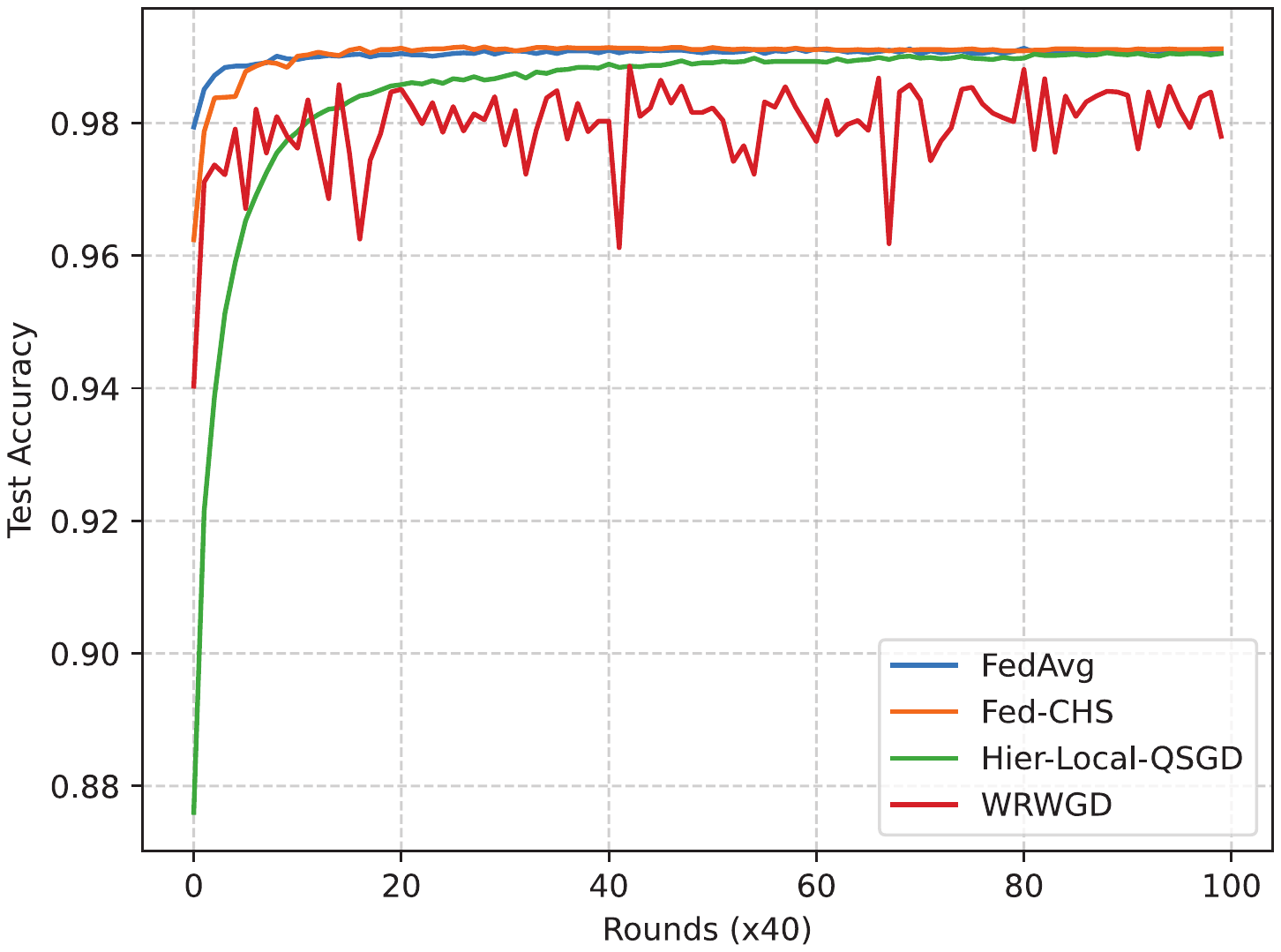}
\label{fig:MNIST-0.6-LENET}
}
\vfill
\caption{Convergence performance of Fed-CHS and baselines in different models and Dirichlet parameters, using MNIST dataset}\label{figs:MNIST}
\end{figure*}
\begin{figure*}[ht]
\centering
\subfigure[MLP and $\lambda=0.3$]{
\includegraphics[width=0.25\linewidth]{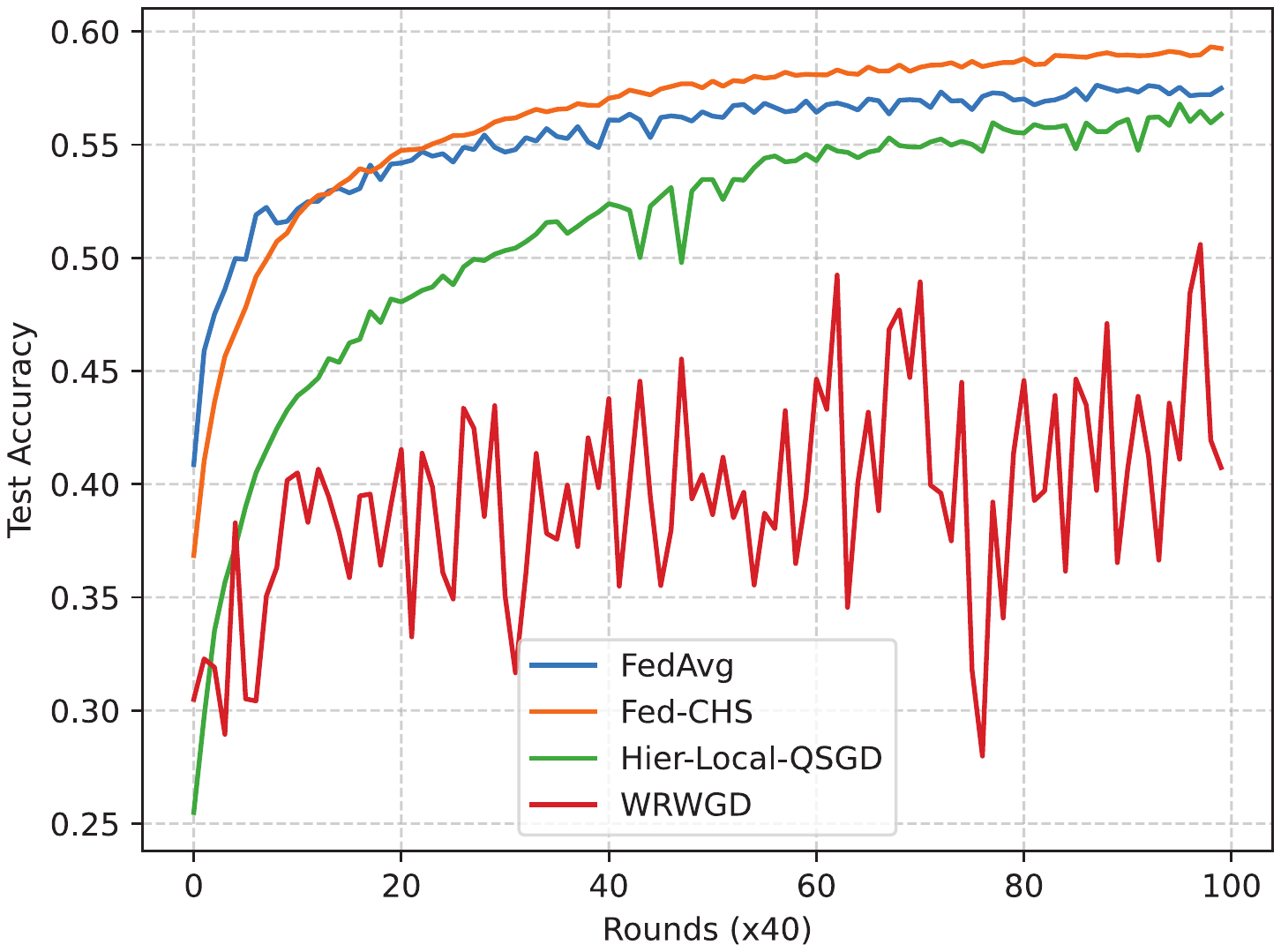}
\label{fig:CIFAR10-0.3-MLP}
}
\!\!\!\!\!\!\!\!\!\!
\subfigure[MLP and $\lambda=0.6$]{
\includegraphics[width=0.25\linewidth]{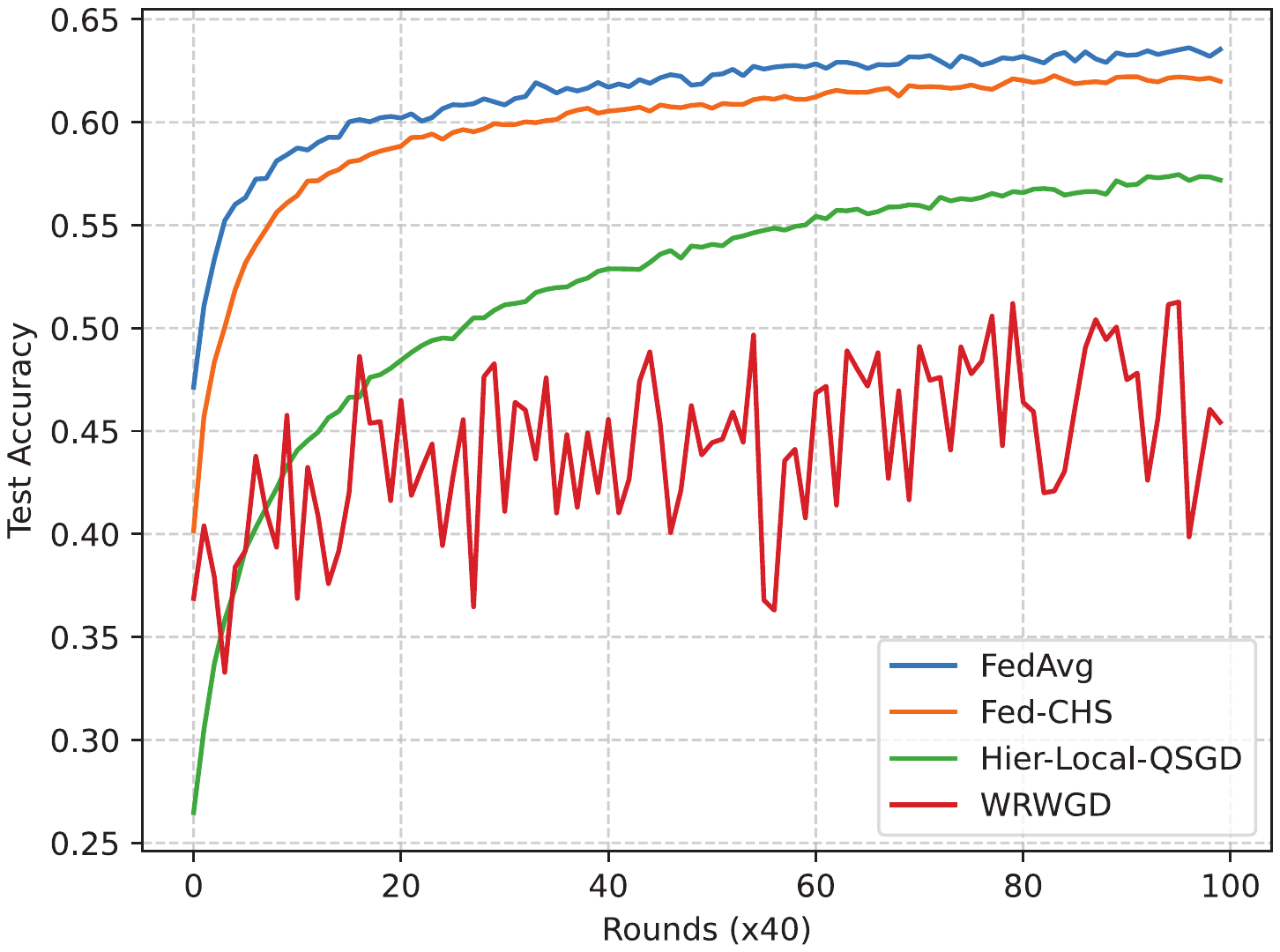}
\label{fig:CIFAR10-0.6-MLP}
}
\!\!\!\!\!\!\!\!\!\!
\subfigure[LENET and $\lambda=0.3$]{
\includegraphics[width=0.25\linewidth]{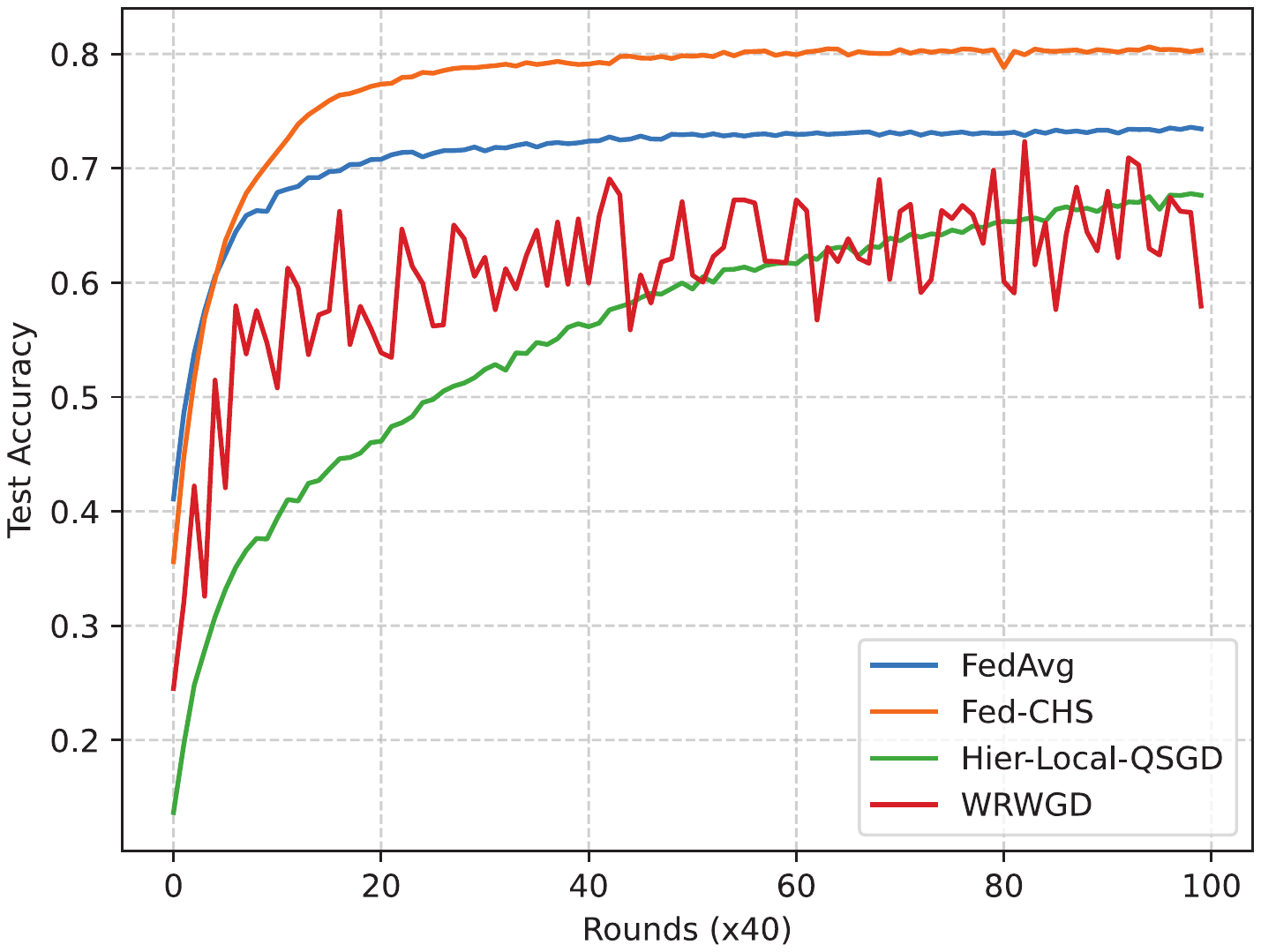}
\label{fig:CIFAR10-0.3-LENET}
}
\!\!\!\!\!\!\!\!\!\!
\subfigure[LENET and $\lambda=0.6$]{
\includegraphics[width=0.25\linewidth]{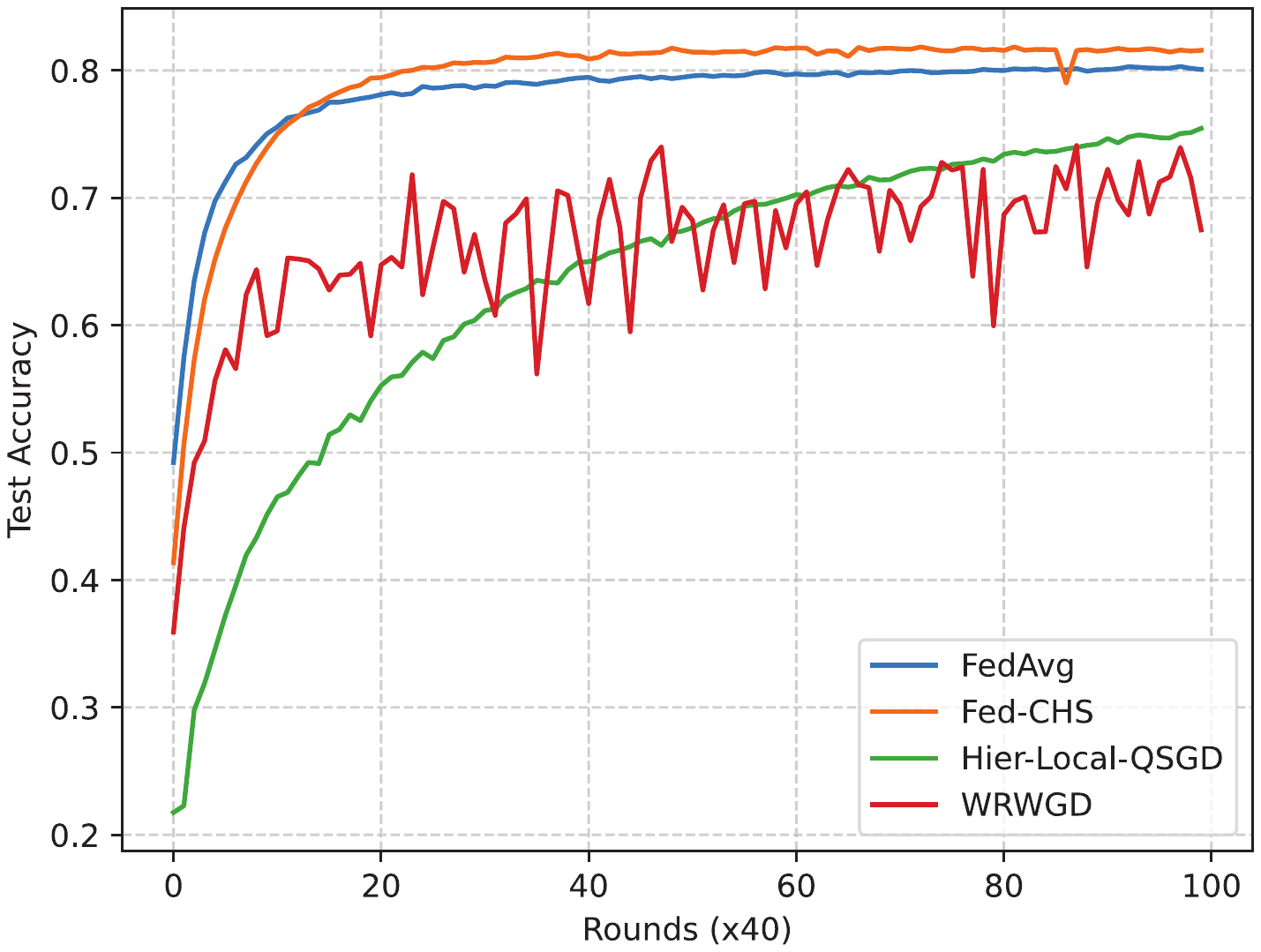}
\label{fig:CIFAR10-0.6-LENET}
}
\vfill
\caption{Convergence performance of Fed-CHS and baselines in different models and Dirichlet parameters, using CIFAR-10 dataset}\label{figs:CIFAR10}
\end{figure*}
\begin{figure*}[ht]
\centering
\subfigure[MLP and $\lambda=0.3$]{
\includegraphics[width=0.25\linewidth]{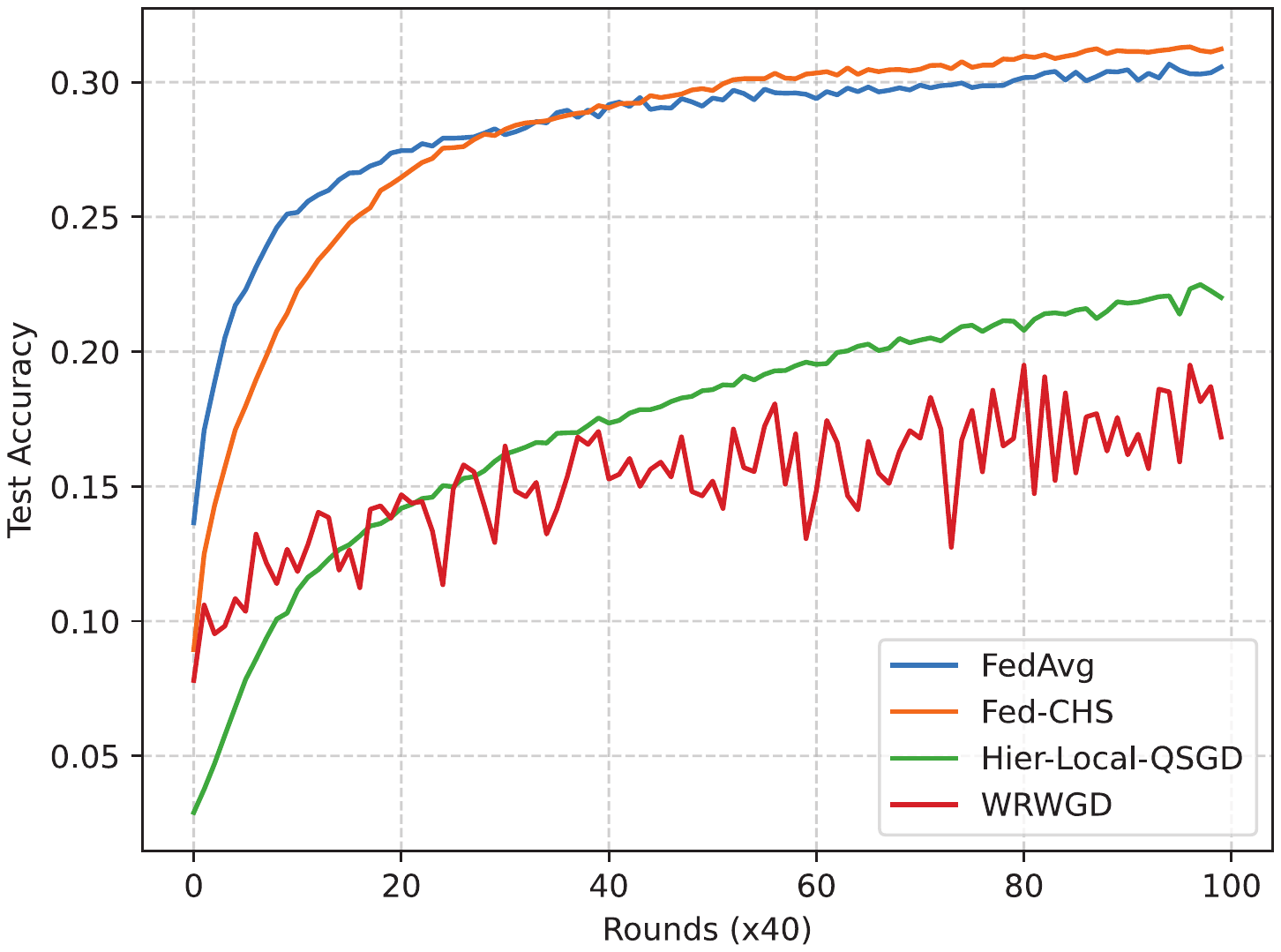}
\label{fig:CIFAR100-0.3-MLP}
}
\!\!\!\!\!\!\!\!\!\!
\subfigure[MLP and $\lambda=0.6$]{
\includegraphics[width=0.25\linewidth]{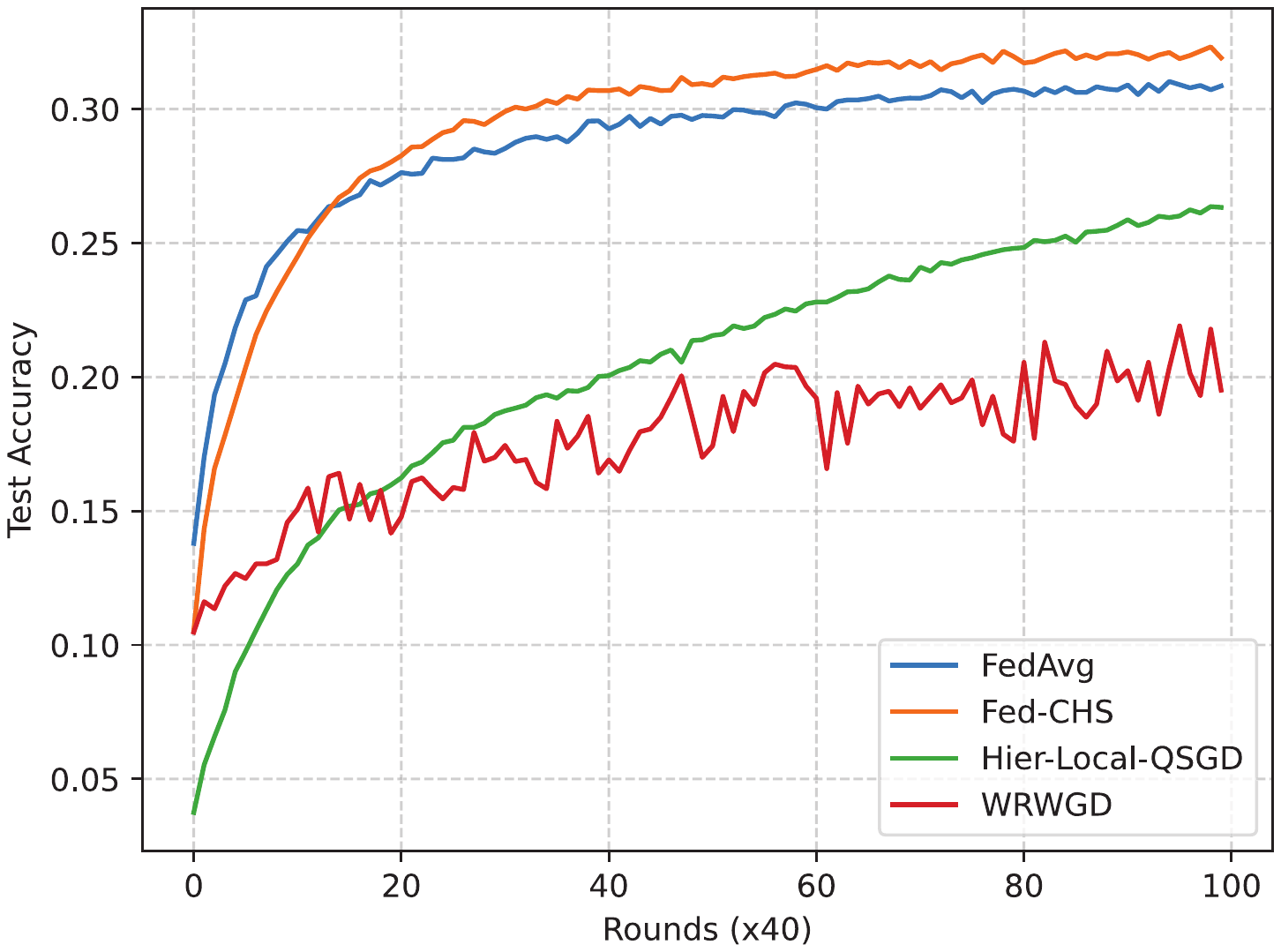}
\label{fig:CIFAR100-0.6-MLP}
}
\!\!\!\!\!\!\!\!\!\!
\subfigure[LENET and $\lambda=0.3$]{
\includegraphics[width=0.25\linewidth]{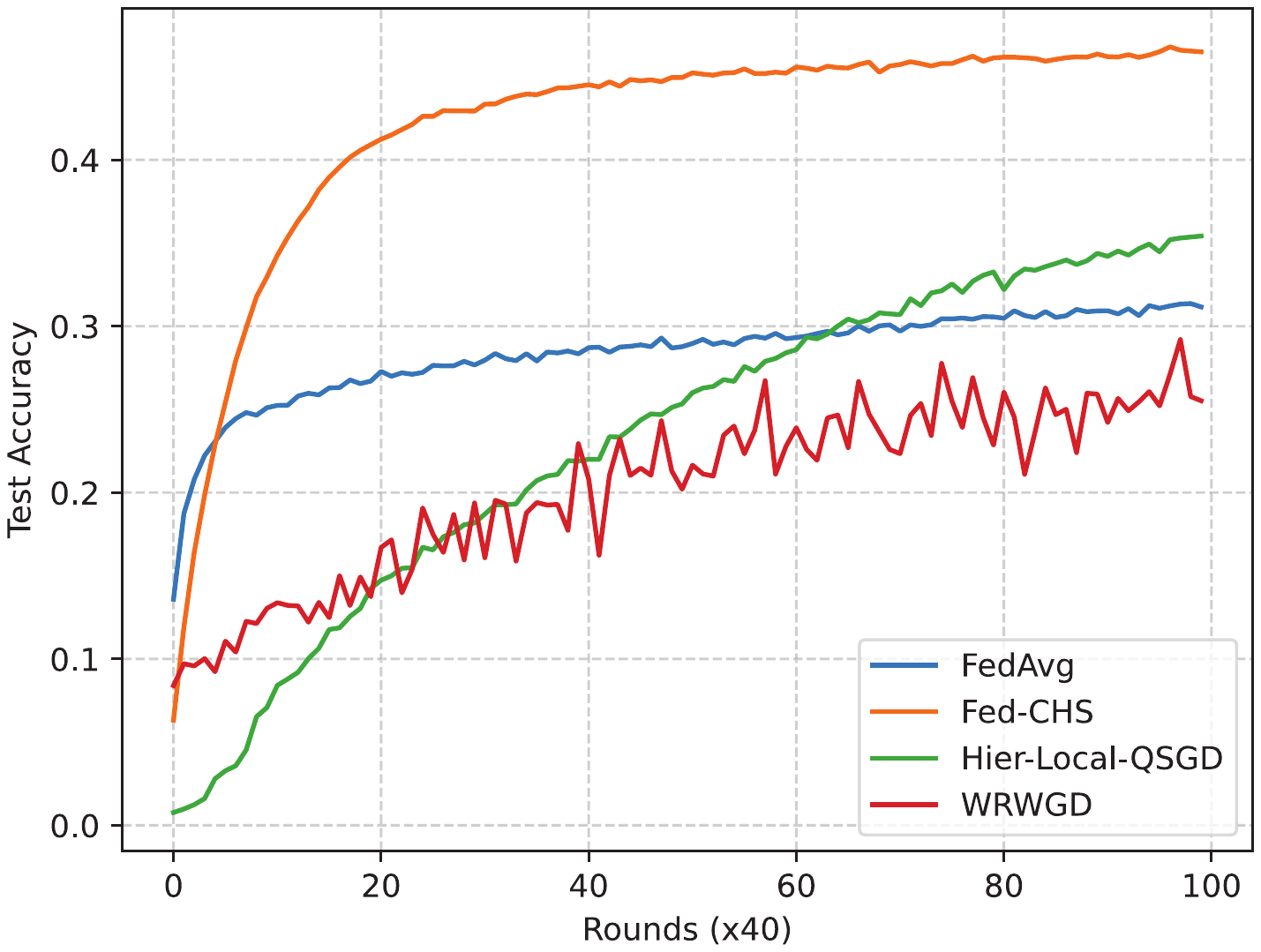}
\label{fig:CIFAR100-0.3-LENET}
}
\!\!\!\!\!\!\!\!\!\!
\subfigure[LENET and $\lambda=0.6$]{
\includegraphics[width=0.25\linewidth]{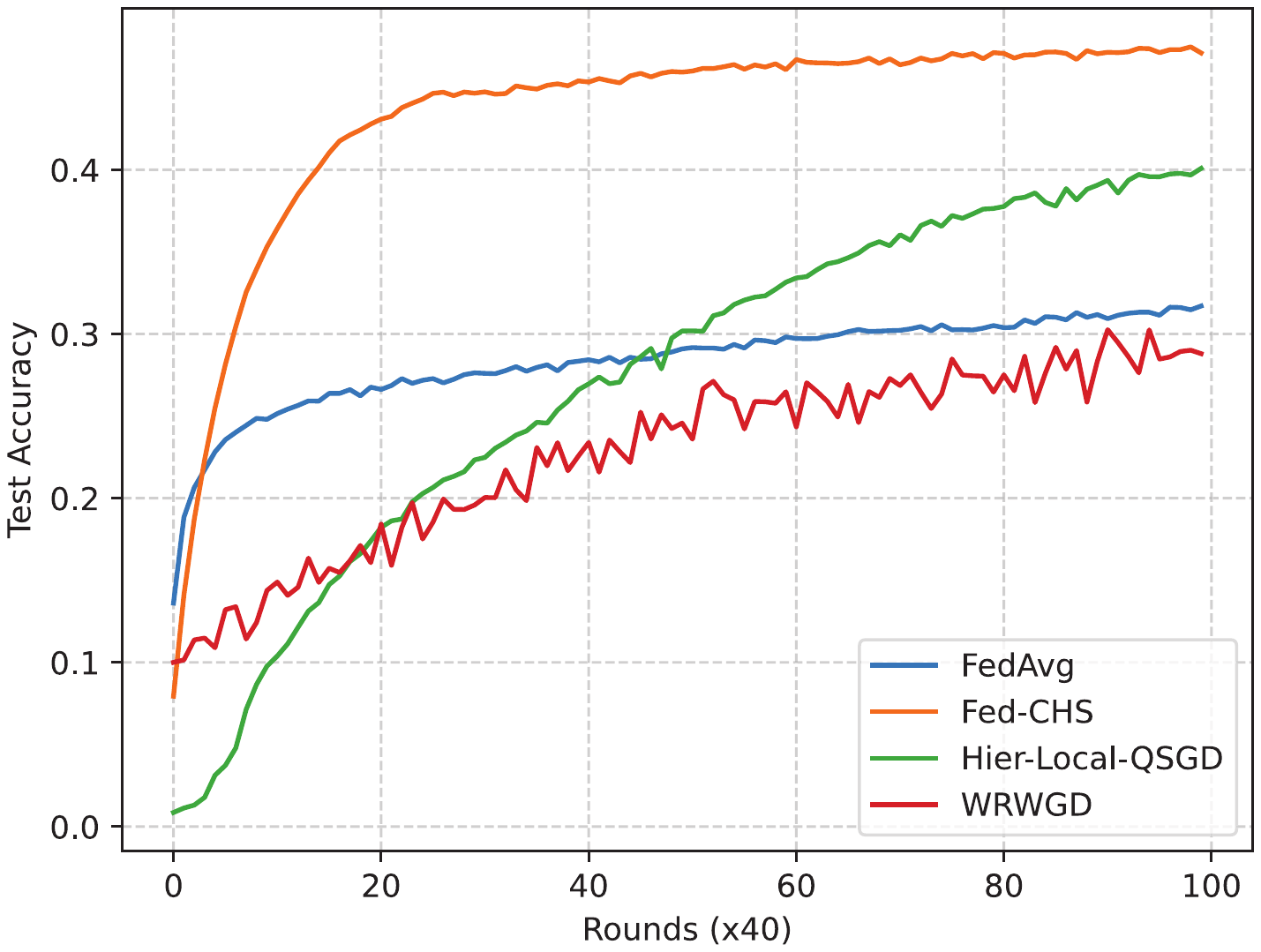}
\label{fig:CIFAR100-0.6-LENET}
}
\vfill
\caption{Convergence performance of Fed-CHS and baselines in different models and Dirichlet parameters, using CIFAR-100 dataset}\label{figs:CIFAR100}
\end{figure*}

\section{Application Scenarios of SFL in Hierarchical Architecture}
\label{apx:app}
{\bf Internet of Vehicles (IoV):} 
To lean a common model from the vehicles in an IoV by FL, a lot of scattered road side units (RSUs) can be leveraged to aggregate the local model parameter of bypassing vehicles through one-hop wireless link \citep{8998397}. In the combinational framework of HFL and SFL, every RSU acts as the ES and the bypassing vehicles are participating clients. Once get the aggregated local model from bypassing vehicles, one RSU can push it to the neighbor RSU for next round of iteration. Through this operation, no data transmission is required from neither any vehicle nor a RSU to a central PS round by round. Lost of traffic burden is saved in the hybrid framework of HFL and SFL.
An illustrative figure of this scenario is given in \cref{fig:scene IOV}
\footnote{Thanks for the courtesy of Freepik.com, which offers parts of the elements for \cref{fig:scene IOV} and \cref{fig:scene LEO}.}.

\begin{figure}
    \centering
    \includegraphics[width=0.7\linewidth]{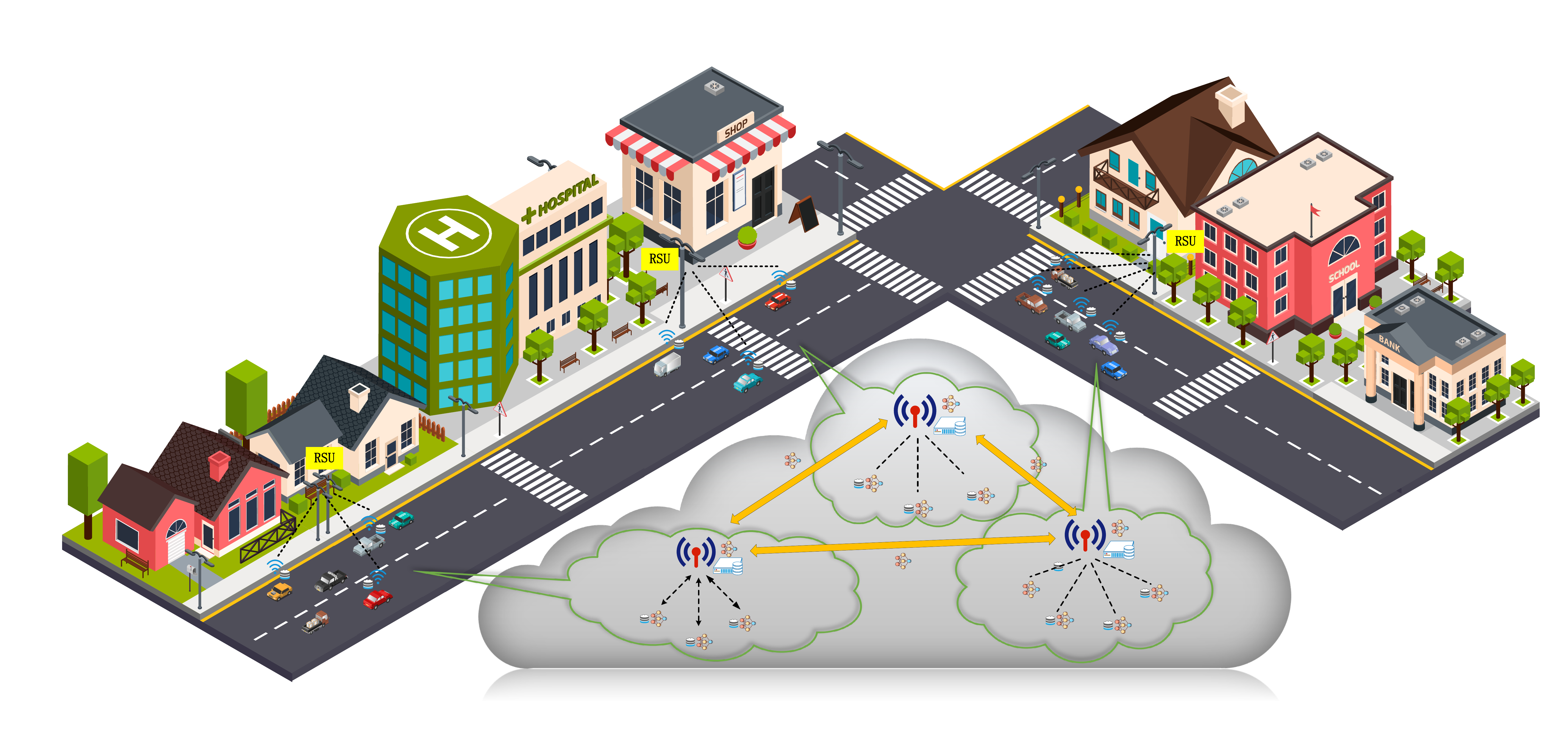}
    \caption{Federated learning used in the scenario of IoV under the combinational framework of HFL and SFL}
    \label{fig:scene IOV}
\end{figure}

{\bf Integrated Low-earth Orbit (LEO) Satellite-Terrestrial Network:} 
Integrated LEO satellite-terrestrial network can provide seamless and reliable content delivery service for the users widely distributed on earth's surface. This type of network such as StarLink has been put into use \citep{ZHAO202294,10113881}. To learn a common model from the surface users in a certain area via the FL technique, local model parameters of multiple surface users can be uploaded to the LEO satellite overhead for aggregation \citep{9099807,9606720}.
Restricted by orbit dynamics, every LEO satellite flies over the sky quickly and would not be able to offer the service of parameter aggregation and broadcasting for long.
To overcome this challenge, we can adopt the combinational framework of HFL and SFL, in which every LEO satellite above the horizon can be taken as one ES and the surface users correspond to participating clients. 
Within such a framework, the LEO satellite with the most up-to-date model parameter but has to go below the horizon soon can hand over its model parameter to the LEO satellite arising in the sky.
Through this way, parameter aggregation and broadcasting can be sustained round by round between the surface users and the LEO satellites over the sky.
This application scenario is demonstrated in \cref{fig:scene LEO}.

\begin{figure}
    \centering
    \includegraphics[width=0.6\linewidth]{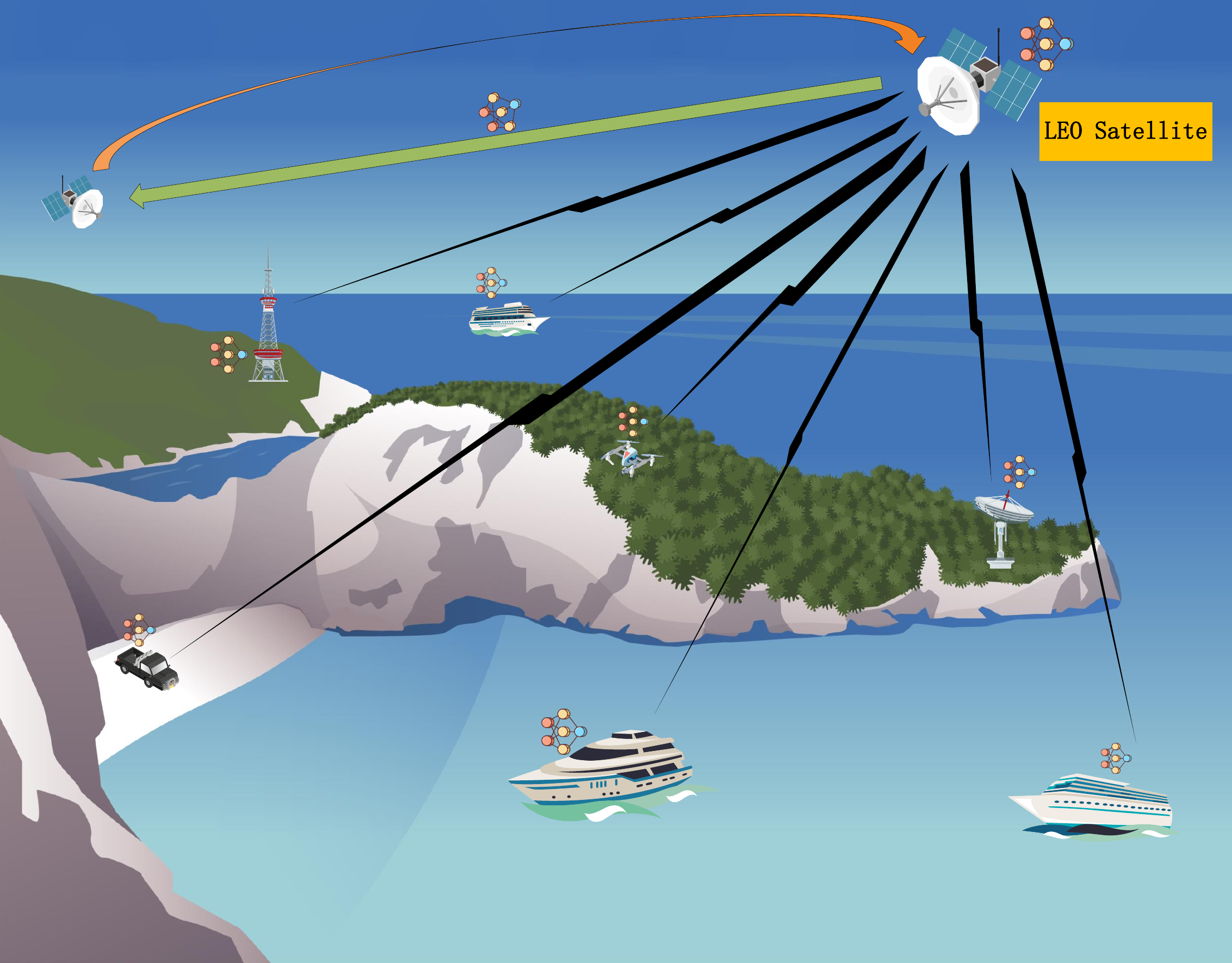}
    \caption{Federated learning under the combinational framework of HFL and SFL for LEO satellite-terrestrial network}
    \label{fig:scene LEO}
\end{figure}

\section{Related Work in Detail}
\label{app:survey}
For the literature in HFL, 
\citet{9699080} proposes to update the aggregated model parameter at each ES to the PS asynchronously, which allows the ES currently having bad link condition with the PS to offload its model parameter only when the link condition turns to be good.   
\citet{tu2020network} leverages fog computing to help the participating clients with weak computation capability. Specifically, when one client cannot update its model parameter at local in time, it can offload its computation task to its neighbor client.
\citet{9705093} builds a multiple-layer (more than two) architecture to cover each participating client in a device-to-device (D2D) enabled network.

In terms of the literature in SFL, 
when the next passing node is selected in fixed order, \citet{wang2022efficient} arranges all the participating clients to be in a ring topology. Differently, \citet{9956084} divides the whole sets of clients into multiple sub-groups, each of which gathers the clients with similar dataset distribution and arranges its clients to be in a ring topology. 
Global model parameter is firstly updated in sequence among the clients in one sub-group and then pushed to the other sub-group.
When the next passing node is selected randomly, 
\citet{8778390} values the neighbor clients, each of which can serve as the candidate of next passing node, with a probability based on the associated Lipschitz constant, which characterizes the smoothness of the local loss function.
\citet{9050511} maps the gradient of each next passing node candidate's local loss function to a selecting probability in a straightforward way, while \citet{10044204} makes this mapping by leveraging  MAB theory.

As of overcoming data heterogeneity, for conventional FL, when the loss function is strongly convex,
\citet{khaled2020tighter} can promise convergence, while 
\citet{li2019convergence} is able to further achieve zero optimality gap.
The convergence speed is mainly at $\mathcal{O}(1/T)$ \citet{li2019convergence} or $\mathcal{O}(1/\sqrt{T})$ \citet{khaled2020tighter}, where $T$ is the number of communication round.
For non-convex and smooth loss functions, prior works can ensure convergence \citet{NEURIPS2022_6db3ea52,pmlr-v180-jhunjhunwala22a,pmlr-v180-das22b,karimireddy2020scaffold} but can hardly achieve stationary point. 
The convergence rate is mainly at $\mathcal{O}(1/T)$ \citep{pmlr-v180-jhunjhunwala22a}, $\mathcal{O}(1/T^{\frac{2}{3}})$ \citep{pmlr-v180-das22b}, or $\mathcal{O}(1/\sqrt{T})$\citep{NEURIPS2022_6db3ea52,karimireddy2020scaffold}.

\section{Assumptions} \label{s:assum}
First we state some general assumptions in the work.
\begin{assumption}
	\label{ass:smooth}
	($L$-smooth).
The loss function $f(w, z_{n,i})$ is $L$-smooth if
	\begin{align}
		f\left(w_1, z_{n,i}\right) \le f\left(w_2, z_{n,i}\right) + \langle \nabla f\left(w_2, z_{n,i}\right),w_1- w_2\rangle + \frac{L}{2}\|w_1-w_2\|^2 
	\end{align}
With Assumption \ref{ass:smooth}, it can be derived that $f_n(w)$, i.e., the local loss function of $n$th client, is also $L$-smooth as they are the linear combination of the $L$-smooth functions $f\left(w, z_{n,i}\right)$ for $z_{n,i}\in \mathcal{D}_n$ \citep{xiao2023communicationefficient}.
Similarly, the function $F_m(w) \triangleq \sum_{n \in \mathcal{N}_m} \gamma_n^m f_n(w)$ and $F(w)$ are all $L$-smooth.
\end{assumption}

\begin{assumption}
	\label{ass:convex}
	($\mu$-strongly Convex).
	The loss function $f(w, z_{n,i})$ is $\mu$-strongly convex if
	\begin{align}
		f\left(w_1, z_{n,i}\right) \geq  f\left(w_2, z_{n,i}\right) + \langle \nabla f\left(w_2, z_{n,i}\right),w_1- w_2\rangle + \frac{\mu}{2}\|w_1-w_2\|^2 
	\end{align}
With Assumption \ref{ass:convex}, similar with the discussion after Assumption \ref{ass:smooth}, $f_n(w)$, $F_m(w)$, and $F(w)$ are all $\mu$-strongly convex as they are the linear combinations of $f\left(w, z_{n,i}\right)$ for $z_{n,i} \in \mathcal{D}_n$ and $n \in \mathcal{N}$ \citep{huang2023achieving,xiao2023communicationefficient}.
\end{assumption}

\begin{assumption}
	\label{ass:gradient}
	(Bounded Stochastic Gradient).
	The norm of gradients is uniformly bounded, i.e.,  
	\begin{equation}
		\left\| \nabla f_{n}\left(w\right)\right\|^2 \le G^2, \forall n \in \mathcal{N}
	\end{equation}
This assumption is general and has been adopted in \citet{NEURIPS2022_6db3ea52,pmlr-v202-li23o,10089406}.
With Assumption \ref{ass:gradient}, it can be derived from Jensen inequality that $\left\| \nabla F_{m}\left(w\right)\right\|^2$ and $\left\| \nabla F\left(w\right)\right\|^2$ are all upper bounded by $G^2$, where $\nabla F_m(w)$ is the gradient of $F_m(w)$ with $w$.
\end{assumption}

\begin{assumption}
	\label{ass:heterogeneity}
	(Bounded Variance and Heterogeneity).
	For $f_n(w)$ and $f(w, z_{n,i})$, the variance of stochastic gradients is assumed to be bounded, i.e.,
 \begin{equation} \label{e:hetero_ass_first}
\!\!\! \left \| \nabla f_{n}\left(w \right) \!-\! \nabla f\left(w, z_{n,i}\right) \right \|^2 \le \sigma^2_{n}, \forall n \in \mathcal{N}, z_{n,i} \in \mathcal{D}_n,
	\end{equation}
where $\sigma_n^2$ represents the local data variance of client $n$ \citep{huang2023achieving,pmlr-v202-li23o}. With \cref{e:hetero_ass_first} and $\forall m \in \mathcal{M}, n \in \mathcal{N}_m$, we can further derive
\begin{gather} \label{e:hetero_ass_second}
		\left \| \nabla f_{n}\left(w \right)- \nabla f\left(w, \xi_{n}\right) \right \|^2 \le \sigma^2_{n}, \forall n \in \mathcal{N},\\
\label{e:hetero_ass_third}
    \left \| \sum_{n \in \mathcal{N}_m} \gamma^m_{n} \left(\nabla f_{n}\left(w \right)- \nabla f \left(w, \xi_{n}\right)\right) \right \|^2 \le \theta^2_m, \forall m \in \mathcal{M}
\end{gather}
where $\theta^2_m = \sum_{n \in \mathcal{N}_m} \gamma^m_{n} \sigma^2_n$.
The inequality in \cref{e:hetero_ass_second,e:hetero_ass_third} define the variance due to data sampling for one client or one cluster.
We also assume
\begin{equation} \label{e:hetero_ass_fifth}
\mathbb{E}  \| \nabla F\left(w\right)-\nabla F_m\left(w\right) \|^2  \le \sigma^2, \forall m \in \mathcal{M}.
\end{equation}
The inequality in \cref{e:hetero_ass_fifth} bounds the heterogeneity between global dataset and the dataset in any cluster.
In summary, this assumption allows the distribution of dataset in any cluster to be heterogeneous from their peers. When $\sigma^2$ goes to zero, the assumed heterogeneity disappears among the clusters.
It is also worthy to note that we still assume the existence of data heterogeneity among all the clients, while we just do not impose any constraint on the extent of them, which is more general.
\end{assumption}

\section{Proof of Theorem \cref{thm:convex}}
\label{apx:convex}
\subsection{Key Lemmas}
To effectively convey our proofs, we need to prove some useful lemmas in advance.
We provide the claim of lemmas in current section, but we defer these lemmas' proof to \cref{proof key lemmas}. 
For convenience, we define {$g_k^t = \sum_{n \in \mathcal{N}_{m(t)}} \gamma_n^{m(t)} \nabla f\left(w^t_k, \xi_{n,k}\right)$}, which represents the weighted sum of the local gradients on ES node $m(t)$ after $k$ round of broadcasting in the cluster.
\begin{lemma}
\label{lemma1}
Define $\overline g_k^t=\mathbb{E}\left[g_k^t\right]=\sum_{n \in \mathcal{N}_{m(t)}} \gamma_n^{m(t)} \nabla f_n\left(w^t_k\right)$, then with \cref{ass:convex}, we have
    \begin{align}
        \label{lemma1-eq}
        \nonumber
        -2\left\langle w^{t-1}-w^*, \sum_{k=0}^{K-1} {\eta_k \overline g_{k}^{t-1}}\right\rangle &\le \sum_{k=0}^{K-1} \left(\frac{1}{K}+\mu\eta_k\right)\left\|w^{t-1}-w^{t-1}_k\right\|^2 \\
        \nonumber
        & + K \sum_{k=0}^{K-1} {\eta_k^2} \sum_{n \in \mathcal{N}_{m(t-1)}} \gamma_n^{m(t-1)} \left\|\nabla f_n\left(w_k^{t-1}\right)\right\|^2 \\
        \nonumber
        &+2\sum_{k=0}^{K-1}\eta_k \sum_{n \in \mathcal{N}_{m(t-1)}} \gamma_n^{m(t-1)} \left(f_n\left(w^*\right)-f_n\left(w^{t-1}_k\right)\right) \\
        &-\frac{\mu}{2}\sum_{k=0}^{K-1} \left\|w^{t-1}-w^*\right\|^2
    \end{align}
\end{lemma}

 \begin{proof}
 	See \cref{proof of lemma 1}
 \end{proof}

\begin{lemma}
\label{lemma2}
Suppose $\eta_k<\frac{1}{2LK}$ for $k\in \mathcal{K}$, then we have
    \begin{align}
        \label{lemma2-eq}
        \nonumber
	-&\sum_{k=0}^{K-1} 2\eta_k\left(1-2LK\eta_k\right) \sum_{n \in \mathcal{N}_{m(t-1)}} \gamma_n^{m(t-1)} \left(f_n\left(w_k^{t-1}\right)-f_n\left(w^*\right)\right) \\
	\nonumber
	& \le \sum_{k=0}^{K-1} 2\eta_k\left(1-2LK\eta_k\right)\left(LK\eta_k-1\right) \sum_{n \in \mathcal{N}_{m(t-1)}} \gamma_n^{m(t-1)} \left(f_n\left(w^{t-1}\right)-f_n\left(w^*\right)\right) \\
	&+\sum_{k=0}^{K-1} 2LK{\eta_k^2} \sum_{n \in \mathcal{N}_{m(t-1)}} \gamma_n^{m(t-1)} \left(f_n\left(w^*\right)-f_n^*\right)+\sum_{k=0}^{K-1} \left(\frac{1}{K}-\mu\eta_k\right) \left\|w^{t-1}_k-w^{t-1}\right\|^2
    \end{align}
\end{lemma}

 \begin{proof}
 	See \cref{proof of lemma 2}
 \end{proof}

\begin{lemma}
	\label{lemma3}
	Assume \cref{ass:heterogeneity}, we have
	\begin{align}
		\label{lemma3-eq}
		\left\|w^{t-1}-w_k^{t-1}\right\|^2\le k\sum_{j=0}^{k-1} {\eta_j^2} G^2+k\sum_{j=0}^{k-1} {\eta_j^2}\theta_{m(t-1)}^2
	\end{align}
\end{lemma}

\begin{proof}
	See \cref{proof of lemma 3}
\end{proof}

\subsection{Completing the Proof of  \cref{thm:convex}}
{Thanks to the smoothness of global loss function $F(w)$ }and the fact that $\nabla F(w^*)=0$, we have
\begin{align}
	\label{th1-1}
	\nonumber
	\mathbb{E}\{F\left(w^t\right)\}-F\left(w^*\right) & \le \mathbb{E}\left\langle \nabla F\left(w^*\right),w^t-w^* \right\rangle+\frac{L}{2}\mathbb{E}\left\|w^t-w^*\right\|^2 \\
	& =\frac{L}{2}\mathbb{E}\left\|w^t-w^*\right\|^2 
\end{align}

Based on the results in \cref{th1-1}, for $\left\|w^t-w^*\right\|$ , there is.
\begin{align}
	\label{th1-2.1}
	\nonumber
	\left\|w^t-w^*\right\|^2&=\left\|w^{t-1}-w^*-\sum_{k=0}^{K-1} {\eta_{k} g_{k}^{t-1}}\right\|^2 \\
	\nonumber
	& =\left\|w^{t-1}-w^*+\sum_{k=0}^{K-1} {\eta_{k} \left(\overline g_{k}^{t-1} - \overline g_{k}^{t-1}-g_{k}^{t-1}\right)}\right\|^2 \\
    \nonumber
	& =\left\|w^{t-1}-w^*-\sum_{k=0}^{K-1} {\eta_{k} \overline g_{k}^{t-1}}\right\|^2+\left\|\sum_{k=0}^{K-1} {\eta_{k} \left(\overline g_{k}^{t-1}-g_{k}^{t-1}\right)}\right\|^2 \\
    & +2\left\langle w^{t-1}-w^*-\sum_{k=0}^{K-1} {\eta_{k} \overline g_{k}^{t-1}},\sum_{k=0}^{K-1} {\eta_{k} \left(\overline g_{k}^{t-1}-g_{k}^{t-1}\right)}\right\rangle
\end{align}

Note that $\mathbb{E}\left[\overline g_{k}^{t-1}-g_{k}^{t-1}\right]=0$, we have
\begin{align}
    \label{th1-2}
    \nonumber
    \mathbb{E}\left\|w^t-w^*\right\|^2&=\mathbb{E}\left\|w^{t-1}-w^*\right\|^2 +\mathbb{E}\left\|\sum_{k=0}^{K-1} {\eta_{k} \left(\overline g_{k}^{t-1}-g_{k}^{t-1}\right)}\right\|^2\\
    &+\mathbb{E}\left\|\sum_{k=0}^{K-1} {\eta_{k} \overline g_{k}^{t-1}}\right\|^2 -2 \mathbb{E}\left\langle w^{t-1}-w^*, \sum_{k=0}^{K-1} {\eta_{k} \overline g_{k}^{t-1}}\right\rangle
\end{align}

By using the Cauchy Schwartz inequality, we have $\sum_{i=1}^n \left\|a_{i}b_{i}\right\|^2 \le \sum_{i=1}^n \left\|a_i\right\|^2 \sum_{i=1}^n \left\|b_i\right\|^2$.
And combing the inequality with \cref{th1-2}, it follows that
\begin{align}
	\label{th1-3}
    \nonumber
	\left\|w^t-w^*\right\|^2 & \le\left\|w^{t-1}-w^*\right\|^2-2\left\langle w^{t-1}-w^*, \sum_{k=0}^{K-1} {\eta_{k} \overline g_{k}^{t-1}}\right\rangle \\
    &+K\sum_{k=0}^{K-1}\eta_{k}^2\left\|\overline g_{k}^{t-1}\right\|^2+K\sum_{k=0}^{K-1}\eta_{k}^2{\theta_{m(t-1)}^2 }
\end{align}
where {$\theta_{m(t)}^2$} can be checked from \cref{ass:heterogeneity}.

From \cref{lemma1}, it follows that
\begin{align}
	\label{th1-4}
	\nonumber
	& \left\|w^t-w^*\right\|^2 \\
    \nonumber
    & \le \left(1-\frac{\mu}{2}\sum_{k=0}^{K-1} \eta_k\right) \left\|w^{t-1}-w^*\right\|^2 + K\sum_{k=0}^{K-1}\eta_{k}^2{\theta_{m(t-1)}^2} +\sum_{k=0}^{K-1} \left(\frac{1}{K}+\mu\eta_k\right)\left\|w^{t-1}-w^{t-1}_k\right\|^2 \\
	& \underbrace{+2K \sum_{k=0}^{K-1} \eta_k^2 \sum_{n \in \mathcal{N}_{m(t-1)}} \gamma_n^{m(t-1)} \left\|\nabla f_n\left(w_k^{t-1}\right)\right\|^2 +2\sum_{k=0}^{K-1}\eta_{k} \sum_{n \in \mathcal{N}_{m(t-1)}} \gamma_n^{m(t-1)} \left(f_n\left(w^*\right)-f_n\left(w^{t-1}_k\right)\right)}_{A_1}
\end{align}

Since $f_n(w)$ is also L-smooth \citep{li2019convergence}, it follows that
\begin{align}
	\label{th1-5}
	\left\|\nabla f_n\left(w_k^{t-1}\right)\right\|^2 \le 2L\left(f_n\left(w_k^{t-1}\right)-f_n^*\right).
\end{align}

Next, we aim to bound $A_1$.
By using \cref{th1-5}, we have
\begin{align}
	\label{th1-6}
	\nonumber
	2&K\sum_{k=0}^{K-1} \eta_k^2 \sum_{n \in \mathcal{N}_{m(t-1)}} {\gamma_n^{m(t-1)}} \left\|\nabla f_n\left(w_k^{t-1}\right)\right\|^2 +2\sum_{k=0}^{K-1}\eta_{k} \sum_{n \in \mathcal{N}_{m(t-1)}} \gamma_n^{m(t-1)} \left(f_n\left(w^*\right)-f_n\left(w^{t-1}_k\right)\right) \\
	\nonumber
	& \le 4LK \sum_{k=0}^{K-1} \eta_k^2 \sum_{n \in \mathcal{N}_{m(t-1)}} {\gamma_n^{m(t-1)}} \left(f_n\left(w_k^{t-1}\right)-f_n^*\right) \\
    \nonumber
    & +2\sum_{k=0}^{K-1}\eta_{k} \sum_{n \in \mathcal{N}_{m(t-1)}} \gamma_n^{m(t-1)} \left(f_n\left(w^*\right)-f_n\left(w^{t-1}_k\right)\right) \\
	\nonumber
	& =\sum_{k=0}^{K-1} 2\eta_k \left(2LK\eta_k-1\right) \sum_{n \in \mathcal{N}_{m(t-1)}} {\gamma_n^{m(t-1)}} \left(f_n\left(w_k^{t-1}\right)-f_n^*\right) \\
    \nonumber
    &+2\sum_{k=0}^{K-1}\eta_{k} \sum_{n \in \mathcal{N}_{m(t-1)}} \gamma_n^{m(t-1)} \left(f_n\left(w^*\right)-f_n^*\right) \\
	\nonumber
	& =-\sum_{k=0}^{K-1} 2\eta_k \left(1-2LK\eta_k\right) \sum_{n \in \mathcal{N}_{m(t-1)}} {\gamma_n^{m(t-1)}} \left(f_n\left(w_k^{t-1}\right)-f_n^*\right) \\
	\nonumber
	& +2\sum_{k=0}^{K-1}\eta_{k} \sum_{n \in \mathcal{N}_{m(t-1)}} {\gamma_n^{m(t-1)}} \left(f_n\left(w^*\right)-f_n^*\right) \\
	\nonumber
	&=\underbrace{-\sum_{k=0}^{K-1} 2\eta_k\left(1-2LK\eta_k\right) \sum_{n \in \mathcal{N}_{m(t-1)}} {\gamma_n^{m(t-1)}} \left(f_n\left(w_k^{t-1}\right)-f_n\left(w^*\right)\right)}_{A_2} \\
	& +4LK\sum_{k=0}^{K-1} {\eta_{k}^2} \sum_{n \in \mathcal{N}_{m(t-1)}} \gamma_n^{m(t-1)} \left(f_n\left(w^*\right)-f_n^*\right) 
\end{align}

From \cref{lemma2} and \cref{lemma3},combining \cref{th1-4} and \cref{th1-6}, it follows that
\begin{align}
	\label{th1-7}
	\nonumber
    \left\|w^t-w^*\right\|^2 & \le \left(1-\frac{\mu}{2}\sum_{k=0}^{K-1} \eta_k\right) \left\|w^{t-1}-w^*\right\|^2 + \sum_{k=0}^{K-1} \left(K\eta_{k}^2+\frac{2k}{K} \sum_{j=0}^{k-1}\eta_j^2\right) {\theta_{m(t-1)}^2}\\
	\nonumber
	&+\sum_{k=0}^{K-1} 2\eta_k\left(1-2LK\eta_k\right)\left(LK\eta_k-1\right) \sum_{n \in \mathcal{N}_{m(t-1)}} {\gamma_n^{m(t-1)}} \left(f_n\left(w^{t-1}\right)-f_n\left(w^*\right)\right) \\
	&+\sum_{k=0}^{K-1} \frac{2k}{K} \sum_{j=0}^{k-1} {\eta_j^2} G^2+\sum_{k=0}^{K-1} 6LK{\eta_k^2} \sum_{n \in \mathcal{N}_{m(t-1)}} {\gamma_n^{m(t-1)}} \left(f_n\left(w^*\right)-f_n^*\right)
\end{align}

Recalling that 
\begin{align}
	& \beta = \frac{\mu}{2} \sum_{k=0}^{K-1} \eta_k, \\
	& \Delta_{m} = \sum_{n \in \mathcal{N}_{m}} \gamma_{n}^{m} \left(f_n\left(w^{t}\right)-f_n\left(w^*\right)\right), \forall m \in \mathcal{M},\\
	& \tau_{m} = \sum_{n \in \mathcal{N}_{m}} \gamma_{n}^{m} \left(f_n\left(w^*\right)-f_n^*\right), \forall m \in \mathcal{M},
\end{align}
through $t$ rounds of iteration, we have
\begin{align}
	\label{th1-8}
	\nonumber
	\left\|w^t-w^*\right\|^2 & \le \left(1-\beta\right)^{t} \left\|w^0-w^*\right\|^2 + \sum_{k=0}^{K-1} \left(K{\eta_{k}^2}+\frac{2k}{K} \sum_{j=0}^{k-1}{\eta_j^2}\right) \sum_{i=0}^{t-1} \left(1-\beta\right)^{i} {\theta_{m(i)}^2} \\
	\nonumber
	&+\sum_{k=0}^{K-1} \frac{2k}{K} \sum_{j=0}^{k-1} {\eta_j^{2}} \sum_{i=0}^{t-1} \left(1-\beta\right)^{i} G^2 \\
	\nonumber
	&+\sum_{k=0}^{K-1} 2\eta_k\left(1-2LK\eta_k\right)\left(LK\eta_k-1\right) \sum_{i=0}^{t-1} \left(1-\beta\right)^{i} \Delta_{m(i)} \\
	&+\sum_{k=0}^{K-1} 6LK{\eta_k^2} \sum_{i=0}^{t-1} \left(1-\beta\right)^{i} \tau_{m(i)}
\end{align}

And combining the \cref{th1-1} and \cref{th1-8}, we bound the $\mathbb{E}\{F\left(w^T\right)\}-F\left(w^*\right)$
\begin{align}
    \nonumber
    \mathbb{E}\{F\left(w^T\right)\}-F\left(w^*\right) & \le \frac{L}{2}\left(1-\beta\right)^{T} \left\|w^0-w^*\right\|^2 + \frac{L}{2}\sum_{k=0}^{K-1} \left(K{\eta_{k}^2}+\frac{2k}{K} \sum_{j=0}^{k-1}{\eta_j^2}\right) \sum_{t=0}^{T-1} \left(1-\beta\right)^{t} {\theta_{m(t)}^2} \\
    \nonumber
	&+\frac{L}{2}\sum_{k=0}^{K-1} 2\eta_k\left(1-2LK\eta_k\right)\left(LK\eta_k-1\right) \sum_{t=0}^{T-1} \left(1-\beta\right)^{t} \Delta_{m(t)} \\
	&+\frac{L}{2}\sum_{k=0}^{K-1} \frac{2k}{K} \sum_{j=0}^{k-1} {\eta_j^{2}} \sum_{t=0}^{T-1} \left(1-\beta\right)^{t} G^2+\frac{L}{2}\sum_{k=0}^{K-1} 6LK{\eta_k^2} \sum_{t=0}^{T-1} \left(1-\beta\right)^{t} \tau_{m(t)}
\end{align}

This completes the proof.

\section{Proof of Theorem \cref{thm:nonconvex}}
\label{apx:nonconvex}
\subsection{Key Lemmas}
We first give some necessary lemmas.
Their proof are deferred to \cref{proof key lemmas}.
Firstly, we still denote $g_k^t = \sum_{n \in \mathcal{N}_{m(t)}} \gamma_n^{m(t)} \nabla f\left(w^t_k, \xi_{n,k}\right)$.
\begin{lemma}
	\label{lemma4}
	Assume \cref{ass:heterogeneity}, it follows that
	\begin{align}
		\label{lemma4-eq}
		\mathbb{E}\left\langle\nabla F\left(w^t\right), \nabla F_{m(t)}\left(w^t\right)-g_k^t\right\rangle \le \frac{1}{4}\mathbb{E}\left\|\nabla F\left(w^t\right)\right\|^2+\theta_{m(t)}^2
	\end{align}
\end{lemma}

 \begin{proof}
 	See \cref{proof of lemma 4}
 \end{proof}

\begin{lemma}
	\label{lemma5}
	According to \cref{ass:heterogeneity}, it follows that
	\begin{align}
		\label{lemma5-eq}
		\mathbb{E}\left\langle\nabla F\left(w^t\right),\nabla F\left(w^t\right)-\nabla F_{m(t)}\left(w^t\right)\right\rangle\le\frac{1}{2}\mathbb{E}\left\|\nabla F\left(w^t\right)\right\|^2-\frac{1}{2}\mathbb{E}\left\|\nabla F_{m(t)}\left(w^t\right)\right\|^2+\frac{1}{2}\sigma^2 
	\end{align}
\end{lemma}

 \begin{proof}
 	See \cref{proof of lemma 5}
 \end{proof}


\subsection{Completing the Proof of \cref{thm:nonconvex}}

In case the loss function is non-convex, we first have
\begin{align}
	\label{th2-1}
    \nonumber
	\left\langle\nabla F\left(w^t\right),g_k^t\right\rangle & =\left\langle\nabla F\left(w^t\right),g_k^t-\nabla F_{m(t)}\left(w^t\right)\right\rangle+\left\langle\nabla F\left(w^t\right),\nabla F_{m(t)}(w^t) - \nabla F\left(w^t\right)\right\rangle \\
    &+\left\langle\nabla F\left(w^t\right), \nabla F\left(w^t\right)\right\rangle
\end{align}
Then there is
\begin{align}
	\label{th2-2}
    \nonumber
	\left\|\nabla F\left(w^t\right)\right\|^2 & = \left\langle\nabla F\left(w^t\right),g_k^t\right\rangle+\langle\nabla F\left(w^t\right),\nabla F_{m(t)}\left(w^t\right)-g_k^t\rangle \\
    &+\langle\nabla F\left(w^t\right),\nabla F\left(w^t\right)-\nabla F_{m(t)}\left(w^t\right) \rangle
\end{align}

From Lemma \ref{lemma4} and Lemma \ref{lemma5}, it follows that
\begin{align}
	\label{th2-3}
	\frac{1}{4}\left\|\nabla F\left(w^t\right)\right\|^2\le\left\langle\nabla F\left(w^t\right),g_k^t\right\rangle+\theta_{m(t)}^2+\frac{1}{2}\sigma^2-\frac{1}{2}\left\|\nabla F_{m(t)}\left(w^t\right)\right\|^2
\end{align}

Therefore, we can get
\begin{align}
	\label{th2-4}
	\left\langle\nabla F\left(w^t\right),g_k^t\right\rangle \ge \frac{1}{4}\left\|\nabla F\left(w^t\right)\right\|^2+\frac{1}{2}\left\|\nabla F_{m(t)}\left(w^t\right)\right\|^2-\left(\theta_{m(t)}^2+\frac{1}{2}\sigma^2\right)
\end{align}

Telescoping $K$ rounds of iterations in cluster $m(t)$, it follows that
\begin{align}
	\label{th2-5}
    \nonumber
	\sum_{k=0}^{K-1} \eta_k\left\langle\nabla F\left(w^t\right),g_k^t\right\rangle &\ge \frac{1}{4}\sum_{k=0}^{K-1} \eta_k\left\|\nabla F\left(w^t\right)\right\|^2+\frac{1}{2}\sum_{k=0}^{K-1} \eta_k\left\|\nabla F_{m(t)}\left(w^t\right)\right\|^2 \\
    &-\sum_{k=0}^{K-1} \eta_k\left(\theta_{m(t)}^2+\frac{1}{2}\sigma^2\right)
\end{align}

Since $F(w)$ is L-smooth, it follows that
\begin{align}
	\label{th2-6}
	\nonumber
	\sum_{k=0}^{K-1} \eta_k\left\langle\nabla F\left(w^t\right),g_k^t\right\rangle &=\left\langle\nabla F\left(w^t\right), w^t-w^{t+1} \right\rangle \\
	\nonumber
	& \le F(w^t)-F(w^{t+1})+\frac{L}{2}\left\|w^t-w^{t+1}\right\|^2 \\
	\nonumber
	& \le F(w^t)-F(w^{t+1})+\frac{L}{2}\left\|\sum_{k=0}^{K-1} \eta_k g_k^t\right\|^2 \\
	\nonumber
	& \le F(w^t)-F(w^{t+1})+\frac{LK}{2}\sum_{k=0}^{K-1} {\eta_k^2} \left\| g_k^t\right\|^2 \\
	& \le F(w^t)-F(w^{t+1})+\frac{LK}{2}\sum_{k=0}^{K-1} {\eta_k^2} \left(\theta_{m(t)}^2+\left\|\overline g_k^t\right\|^2\right)
\end{align}

Combining \cref{th2-5} and \cref{th2-6}, it follows that
\begin{align}
	\label{th2-7}
	\nonumber
	\frac{1}{4}\sum_{k=0}^{K-1} \eta_k \left\|\nabla F\left(w^t\right)\right\|^2 & \le F\left(w^t\right)-F\left(w^{t+1}\right)+\left(\frac{LK}{2}\sum_{k=0}^{K-1} {\eta_k^2} + \sum_{k=0}^{K-1} {\eta_k}\right)\theta_{m(t)}^2 &\\
	& +\sum_{k=0}^{K-1} \left(\frac{LK}{2} {\eta_k^2}-\frac{1}{2} \eta_k\right)\left\|\nabla F_{m(t)}\left(w^t\right)\right\|^2+\frac{1}{2}\sum_{k=0}^{K-1} \eta_k\sigma^2
\end{align}

Recalling that $\eta_k \le \frac{1}{LK}$,
divide both sides of the inequality (\ref{th2-7}) by $\frac{1}{4}\sum_{k=0}^{K-1} \eta_k$, then we have
\begin{align}
	\label{th2-8}
	\left\|\nabla F\left(w^t\right)\right\|^2 \le \frac{4\left(F\left(w^t\right)-F\left(w^{t+1}\right)\right)}{\sum_{k=0}^{K-1} \eta_k}+\left(\frac{2LK\sum_{k=0}^{K-1} {\eta_k^2}}{\sum_{k=0}^{K-1} \eta_k} + 4\right)\theta_{m(t)}^2+2\sigma^2
\end{align}

Given that the minimum of $||\nabla F(w^t)||^2$ for $t=0, 1, ..., T-1$ is upper bounded by the average of them, we can bound the minimum of $||\nabla F(w^t)||^2$ for $t=0, 1, ..., T-1$ as follows
\begin{flalign}
	\label{th2-9}
	\nonumber
	&\mathbb{E}\left[\frac{1}{T}\sum_{t=0}^{T-1} \left\|\nabla F\left(w^t\right)\right\|^2\right] \\
    \nonumber
    & \le \frac{1}{T}\sum_{t=0}^{T-1} \frac{4\left(F\left(w^t\right)-F\left(w^{t+1}\right)\right)}{\sum_{k=0}^{K-1} \eta_k}+\left(\frac{2LK\sum_{k=0}^{K-1} {\eta_k^2}}{\sum_{k=0}^{K-1} \eta_k} + 4\right) \frac{1}{T}\sum_{t=0}^{T-1} \theta_{m(t)}^2+\frac{2}{T}\sum_{t=0}^{T-1}\sigma^2 \\
	& \le \frac{4\left[F\left(w^0\right)-F\left(w^{t}\right)\right]}{T\sum_{k=0}^{K-1} \eta_k} + \left(\frac{2LK\sum_{k=0}^{K-1} {\eta_k^2}}{\sum_{k=0}^{K-1} \eta_k} + 4\right)\frac{1}{T}\sum_{t=0}^{T-1} \theta_{m(t)}^2+\frac{2}{T}\sum_{t=0}^{T-1}\sigma^2
\end{flalign}

Since $F(w^*)$ be the global minimum of the loss function $F(w)$, then there is 
\begin{equation}
	\label{th2-10}
	\begin{split}
		F\left(w^0\right)-F\left(w^{t}\right) \le F\left(w^0\right)-F\left(w^*\right)
	\end{split}
\end{equation}

Combining \cref{th2-9} and \cref{th2-10}, we have
\begin{align}
	\label{th2-11}
	\mathbb{E}\left[\frac{1}{T}\sum_{t=0}^{T-1} \left\|\nabla F\left(w^t\right)\right\|^2\right]\le\frac{4\left[F\left(w^0\right)-F\left(w^*\right)\right]}{T\sum_{k=0}^{K-1} \eta_k}+\left(\frac{2LK\sum_{k=0}^{K-1} {\eta_k^2}}{\sum_{k=0}^{K-1} \eta_k} + 4\right)\frac{1}{T}\sum_{t=0}^{T-1} \theta_{m(t)}^2 +\frac{2}{T}\sum_{t=0}^{T-1} \sigma^2
\end{align}

This completes the proof.

\section{Defered Proof of Lemmas}
\label{proof key lemmas}
\subsection{Proof of  \cref{lemma1}}
\label{proof of lemma 1}
Splitting and expanding the left-hand side of \cref{lemma1-eq}, we can get
\begin{align}
	\label{pf1-1}
	\nonumber
	-2\left\langle w^{t-1}-w^*, \sum_{k=0}^{K-1} \eta_k \overline g_k^{t-1}\right\rangle &=-2\sum_{k=0}^{K-1} \eta_k\left\langle w^{t-1}-w^*, \sum_{n \in \mathcal{N}_{m(t-1)}} \gamma_n^{m(t-1)}\nabla f_n\left(w_k^{t-1}\right)\right\rangle \\
    \nonumber
	& =\underbrace{-2\sum_{k=0}^{K-1} \eta_{k}\left\langle w^{t-1}-w^{t-1}_k, \sum_{n \in \mathcal{N}_{m(t-1)}} \gamma_n^{m(t-1)} \nabla f_n\left(w_k^{t-1}\right)\right\rangle}_{C_1} \\
    &-\underbrace{2\sum_{k=0}^{K-1} \eta_{k} \left\langle w^{t-1}_k-w^*, \sum_{n \in \mathcal{N}_{m(t-1)}} \gamma_n^{m(t-1)} \nabla f_n\left(w_k^{t-1}\right)\right\rangle}_{C_2}
\end{align}

By using Cauchy Schwartz and AM-GM inequality \citep{9916263}, we can bound $C_1$
\begin{align}
	\label{pf1-2}
	\nonumber
	-&2\sum_{k=0}^{K-1} \eta_{k}\left\langle w^{t-1}-w^{t-1}_k, \sum_{n \in \mathcal{N}_{m(t-1)}} \gamma_n^{m(t-1)} \nabla f_n\left(w_k^{t-1}\right)\right\rangle \\
    \nonumber
    &\le \sum_{k=0}^{K-1} \eta_k\left[\frac{1}{K\eta_k}\left\|w^{t-1}-w^{t-1}_k\right\|^2+K\eta_k\left\|\sum_{n \in \mathcal{N}_{m(t-1)}}  \gamma_n^{m(t-1)} \nabla f_n\left(w_k^{t-1}\right)\right\|^2\right] \\
	& =\frac{1}{K} \sum_{k=0}^{K-1} \left\|w^{t-1}-w^{t-1}_k\right\|^2+K \sum_{k=0}^{K-1} {\eta_k^2} \sum_{n \in \mathcal{N}_{m(t-1)}} \gamma_n^{m(t-1)} \left\|\nabla f_n\left(w_k^{t-1}\right)\right\|^2
\end{align} 

Because of the strong convexity of $f_n(w)$ for $n\in \mathcal{N}_{m(t-1)}$, we have
\begin{align}
	\label{pf1-3}
	\nonumber
    &-2\sum_{k=0}^{K-1} \eta_{k} \left\langle w^{t-1}_k-w^*, \sum_{n \in \mathcal{N}_{m(t-1)}} \gamma_n^{m(t-1)} \nabla f_n\left(w_k^{t-1}\right)\right\rangle \\
    &\le 2\sum_{k=0}^{K-1}\eta_{k} \sum_{n \in \mathcal{N}_{m(t-1)}} \gamma_n^{m(t-1)} \left(f_n\left(w^*\right)-f_n\left(w^{t-1}_k\right)\right)-\mu\sum_{k=0}^{K-1} \eta_k\left\|w^{t-1}_k-w^*\right\|^2
\end{align}

By using Cauchy Schwartz inequality, we have
\begin{align}
	\label{pf1-4}
	-\mu\sum_{k=0}^{K-1} \eta_k\left\|w^{t-1}_k-w^*\right\|^2 \le \mu\sum_{k=0}^{K-1} \eta_k\left[\left\|w^{t-1}_k-w^{t-1}\right\|^2-\frac{1}{2}\left\|w^{t-1}-w^*\right\|^2\right] 
\end{align}

We combine \cref{pf1-1}, \cref{pf1-2}, \cref{pf1-3} and \cref{pf1-4}, it follows that
\begin{align}
	\label{pf1-5}
	\nonumber
	& -2\left\langle w^{t-1}-w^*, \sum_{k=0}^{K-1} {\eta_{k} \overline g_{k}^{t-1}}\right\rangle \\
    \nonumber
    &\le \sum_{k=0}^{K-1} \left(\frac{1}{K}+\mu\eta_k\right)\left\|w^{t-1}-w^{t-1}_k\right\|^2 + K \sum_{k=0}^{K-1} {\eta_k^2} \sum_{n \in \mathcal{N}_{m(t-1)}} \gamma_n^{m(t-1)} \left\|\nabla f_n\left(w_k^{t-1}\right)\right\|^2 \\
	&+2\sum_{k=0}^{K-1}\eta_{k} \sum_{n \in \mathcal{N}_{m(t-1)}} \gamma_n^{m(t-1)} (f_n\left(w^*\right)-f_n\left(w^{t-1}_k\right))-\frac{\mu}{2}\sum_{k=0}^{K-1} \left\|w^{t-1}-w^*\right\|^2
\end{align}

This completes the proof.

\subsection{Proof of  \cref{lemma2}}
\label{proof of lemma 2}
We can split \cref{lemma2-eq} into two terms, it follows that
\begin{align}
    \label{pf2-1}
    \nonumber
    -&\sum_{k=0}^{K-1} 2\eta_k\left(1-2LK\eta_k\right) \sum_{n \in \mathcal{N}_{m(t-1)}} \gamma_n^{m(t-1)} \left(f_n\left(w_k^{t-1}\right)-f_n\left(w^*\right)\right) \\
    \nonumber
    & = \underbrace{-\sum_{k=0}^{K-1} 2\eta_k\left(1-2LK\eta_k\right) \sum_{n \in \mathcal{N}_{m(t-1)}} \gamma_n^{m(t-1)} \left(f_n\left(w_k^{t-1}\right)-f_n\left(w^{t-1}\right)\right)}_{C_3} \\
    & -\sum_{k=0}^{K-1} 2\eta_k\left(1-2LK\eta_k\right) \sum_{n \in \mathcal{N}_{m(t-1)}} \gamma_n^{m(t-1)} \left(f_n\left(w^{t-1}\right)-f_n\left(w^*\right)\right)
\end{align}

Attributed to the strong convexity of $f_n(w)$, we have
\begin{align}
	\label{pf2-2}
	f_n\left(w_k^{t-1}\right)-f_n\left(w^{t-1}\right) \ge \left\langle \nabla f_n\left(w^{t-1}\right), w^{t-1}_k-w^{t-1} \right\rangle + \frac{\mu}{2}\left\|w^{t-1}_k-w^{t-1}\right\|^2
\end{align}

Combing \cref{pf2-2} with $C_3$, we have
\begin{align}
	\label{pf2-3}
	\nonumber
	-&\sum_{k=0}^{K-1} 2\eta_k\left(1-2LK\eta_k\right) \sum_{n \in \mathcal{N}_{m(t-1)}} \gamma_n^{m(t-1)} \left(f_n\left(w_k^{t-1}\right)-f_n\left(w^{t-1}\right)\right) \\
	\nonumber
	& \le -\mu\sum_{k=0}^{K-1} \eta_k\left(1-2LK\eta_k\right) \sum_{n \in \mathcal{N}_{m(t-1)}} \gamma_n^{m(t-1)} \left\|w^{t-1}_k-w^{t-1}\right\|^2 \\
	& - \sum_{k=0}^{K-1} 2\eta_k\left(1-2LK\eta_k\right) \sum_{n \in \mathcal{N}_{m(t-1)}} \gamma_n^{m(t-1)} \left\langle \nabla f_n\left(w^{t-1}\right), w^{t-1}_k-w^{t-1} \right\rangle 
\end{align}

By using Cauchy Schwartz inequality, we have
\begin{align}
	\label{pf2-4}
	\nonumber
	-& \sum_{k=0}^{K-1} 2\eta_k\left(1-2LK\eta_k\right) \sum_{n \in \mathcal{N}_{m(t-1)}} \gamma_n^{m(t-1)} \left\langle \nabla f_n\left(w^{t-1}\right), w^{t-1}_k-w^{t-1} \right\rangle \\
	& \le \sum_{k=0}^{K-1} 2\eta_k\left(1-2LK\eta_k\right) \sum_{n \in \mathcal{N}_{m(t-1)}} \gamma_n^{m(t-1)} \left(\frac{K\eta_k}{2}\left\|f_n\left(w^{t-1}\right)\right\|^2+\frac{1}{2K\eta_k}\left\| w^{t-1}_k-w^{t-1}\right\|^2\right)
\end{align}

Combining \cref{pf2-3} and \cref{pf2-4}, we can bound $C_3$
\begin{align}
	\label{pf2-5}
	\nonumber
	-&\sum_{k=0}^{K-1} 2\eta_k\left(1-2LK\eta_k\right) \sum_{n \in \mathcal{N}_{m(t-1)}} \gamma_n^{m(t-1)} \left(f_n\left(w_k^{t-1}\right)-f_n\left(w^{t-1}\right)\right) \\
	\nonumber
	& \le\sum_{k=0}^{K-1} 2LK{\eta_k^2}\left(1-2LK\eta_k\right) \sum_{n \in \mathcal{N}_{m(t-1)}} \gamma_n^{m(t-1)} \left(f_n\left(w^{t-1}\right)-f_n^*\right) \\
	& +\sum_{k=0}^{K-1} \left(\frac{1}{K}-\mu\eta_k\right)\left(1-2LK\eta_k\right) \sum_{n \in \mathcal{N}_{m(t-1)}} \gamma_n^{m(t-1)} \left\|w^{t-1}_k-w^{t-1}\right\|^2
\end{align}

Next, by combining \cref{pf2-1} and \cref{pf2-5}, it follows that
\begin{align}
	\label{pf2-6}
	\nonumber
	-&\sum_{k=0}^{K-1} 2\eta_k\left(1-2LK\eta_k\right) \sum_{n \in \mathcal{N}_{m(t-1)}} \gamma_n^{m(t-1)} \left(f_n\left(w_k^{t-1}\right)-f_n\left(w^*\right)\right) \\
	\nonumber
	& \le\sum_{k=0}^{K-1} 2LK{\eta_k^2}\left(1-2LK\eta_k\right) \sum_{n \in \mathcal{N}_{m(t-1)}} \gamma_n^{m(t-1)} \left(f_n\left(w^{t-1}\right)-f_n^*\right) \\
	\nonumber
	& \underbrace{-\sum_{k=0}^{K-1} 2\eta_k\left(1-2LK\eta_k\right) \sum_{n \in \mathcal{N}_{m(t-1)}} \gamma_n^{m(t-1)} \left(f_n\left(w^{t-1}\right)-f_n\left(w^*\right)\right)}_{C_4} \\
	& +\sum_{k=0}^{K-1} \left(\frac{1}{K}-\mu\eta_k\right)\left(1-2LK\eta_k\right) \sum_{n \in \mathcal{N}_{m(t-1)}} \gamma_n^{m(t-1)} \left\|w^{t-1}_k-w^{t-1}\right\|^2
\end{align}

Note that
\begin{align}
	\label{pf2-7}
	\nonumber
	C_4 &\le \sum_{k=0}^{K-1} 2\eta_k\left(1-2LK\eta_k\right)\left(LK\eta_k-1\right) \sum_{n \in \mathcal{N}_{m(t-1)}} \gamma_n^{m(t-1)} \left(f_n\left(w^{t-1}\right)-f_n\left(w^*\right)\right) \\
	& +\sum_{k=0}^{K-1} 2LK{\eta_k^2}\left(1-2LK\eta_k\right) \sum_{n \in \mathcal{N}_{m(t-1)}} \gamma_n^{m(t-1)} \left(f_n\left(w^*\right)-f_n^*\right)
\end{align}

Combining \cref{pf2-6} and \cref{pf2-7}, we have
\begin{align}
	\label{pf2-8}
	\nonumber
	-&\sum_{k=0}^{K-1} 2\eta_k\left(1-2LK\eta_k\right) \sum_{n \in \mathcal{N}_{m(t-1)}} \gamma_n^{m(t-1)} \left(f_n\left(w_k^{t-1}\right)-f_n\left(w^*\right)\right) \\
	\nonumber
	&\le \sum_{k=0}^{K-1} 2\eta_k\left(1-2LK\eta_k\right)\left(LK\eta_k-1\right) \sum_{n \in \mathcal{N}_{m(t-1)}} \gamma_n^{m(t-1)} \left(f_n\left(w^{t-1}\right)-f_n\left(w^*\right)\right) \\
	\nonumber
	&+\sum_{k=0}^{K-1} 2LK{\eta_k^2}\left(1-2LK\eta_k\right) \sum_{n \in \mathcal{N}_{m(t-1)}} \gamma_n^{m(t-1)} \left(f_n\left(w^*\right)-f_n^*\right)\\
	&+\sum_{k=0}^{K-1} \left(\frac{1}{K}-\mu\eta_k\right)\left(1-2LK\eta_k\right) \sum_{n \in \mathcal{N}_{m(t-1)}} \gamma_n^{m(t-1)} \left\|w^{t-1}_k-w^{t-1}\right\|^2
\end{align}

Given that $\eta_k<\frac{1}{2LK}$ for $k\in \mathcal{K}$, it follows that
\begin{align}
	\label{pf2-9}
	\nonumber
	-&\sum_{k=0}^{K-1} 2\eta_k\left(1-2LK\eta_k\right) \sum_{n \in \mathcal{N}_{m(t-1)}} \gamma_n^{m(t-1)} \left(f_n\left(w_k^{t-1}\right)-f_n\left(w^*\right)\right) \\
	\nonumber
	& \le \sum_{k=0}^{K-1} 2\eta_k\left(1-2LK\eta_k\right)\left(LK\eta_k-1\right) \sum_{n \in \mathcal{N}_{m(t-1)}} \gamma_n^{m(t-1)} \left(f_n\left(w^{t-1}\right)-f_n\left(w^*\right)\right) \\
	\nonumber
	&+\sum_{k=0}^{K-1} 2LK{\eta_k^2} \sum_{n \in \mathcal{N}_{m(t-1)}} \gamma_n^{m(t-1)} \left(f_n\left(w^*\right)-f_n^*\right)\\
	&+\sum_{k=0}^{K-1} \left(\frac{1}{K}-\mu\eta_k\right) \left\|w^{t-1}_k-w^{t-1}\right\|^2
\end{align}

This completes the proof.

\subsection{Proof of  \cref{lemma3}}
\label{proof of lemma 3}

Notice that $w_k^{t-1}=w^{t-1}-\sum_{j=0}^{k-1} \eta_j g_j^{t-1}$, then
\begin{align}
	\label{prf3-1}
	\nonumber
	\left\|w^{t-1}-w_k^{t-1}\right\|^2 &\le \left\|\sum_{j=0}^{k-1} \eta_j g_j^{t-1}\right\|^2 \\
	\nonumber
	& \le k\sum_{j=0}^{k-1} {\eta_j^2 }\left\|g_j^{t-1}\right\|^2 \\
	\nonumber
	& \le k\sum_{j=0}^{k-1} {\eta_j^2}\left\|g_j^{t-1}-\overline g_j^{t-1}+\overline g_j^{t-1}\right\|^2 \\
	\nonumber
	& \le k\sum_{j=0}^{k-1} {\eta_j^2}\left\|\overline g_j^{t-1}\right\|^2 + 2k\sum_{j=0}^{k-1} {\left\langle\overline g_j^{t-1},g_j^{t-1}-\overline g_j^{t-1}\right\rangle} + k\sum_{j=0}^{k-1} {\eta_j^2}\left\|g_j^{t-1}-\overline g_j^{t-1}\right\|^2 \\
	& \le k\sum_{j=0}^{k-1} {\eta_j^2}\left\|\overline g_j^{t-1}\right\|^2 +k\sum_{j=0}^{k-1} {\eta_j^2}\left\|g_j^{t-1}-\overline g_j^{t-1}\right\|^2 
\end{align}

Based on the \cref{ass:gradient} and \cref{ass:heterogeneity}, we have
\begin{align}
	\label{prf3-2}
	\left\|w^{t-1}-w_k^{t-1}\right\|^2\le k\sum_{j=0}^{k-1} {\eta_j^2} G^2+k\sum_{j=0}^{k-1} {\eta_j^2}\theta_{m(t-1)}^2
\end{align}

This completes the proof.

\subsection{Proof of \cref{lemma4}}
\label{proof of lemma 4}
By using Cauchy-Schwartz inequality and AM-GM inequality \citep{9916263}, it follows that
\begin{align}
	\label{prf4-1}
	\left\langle\nabla F\left(w^t\right), \nabla F_{m(t)}\left(w^t\right)-g_k^t\right\rangle \le \frac{1}{4}\mathbb{E}\left\|\nabla F\left(w^t\right)\right\|^2+\left\|\nabla F_{m(t)}\left(w^t\right)-g_k^t\right\|^2 
\end{align}

From \cref{ass:heterogeneity}, we have
\begin{align}
	\label{prf4-2}
	\mathbb{E}\left\|\nabla F_{m(t)}\left(w^t\right)-g_k^t\right\|^2 \le \mathbb{E}\left\| \sum_{n \in \mathcal{N}_{m(t)}} \gamma^{m(t)}_{n} \left(\nabla f_{n}\left(w \right)- \nabla f \left(w, \xi_{n,k}\right)\right) \right\|^2 \le \theta_{m(t)}^2
\end{align}

Combining \cref{prf4-1} and \cref{prf4-2}, it follows that
\begin{align}
	\label{prf4-4}
	\mathbb{E}\left\langle\nabla F\left(w^t\right), \nabla F_{m(t)}\left(w^t\right)-g_k^t\right\rangle\le \frac{1}{4}\mathbb{E}\left\|\nabla F\left(w^t\right)\right\|^2+\theta_{m(t)}^2
\end{align}

This completes the proof.

\subsection{Proof of \cref{lemma5}}
\label{proof of lemma 5}
For any two vectors $a$ and $b$, we have $2a^Tb=\left\|a\right\|^2+\left\|b\right\|^2-\left\|a-b\right\|^2$. 
Therefore we can split $\left\langle\nabla F\left(w^t\right),\nabla F\left(w^t\right)-\nabla F_{m(t)}\left(w^t\right)\right\rangle$ into three terms
\begin{align}
	\label{prf5-1}
    \nonumber
	\mathbb{E}\left\langle\nabla F\left(w^t\right),\nabla F\left(w^t\right)-\nabla F_{m(t)}\left(w^t\right)\right\rangle&=\frac{1}{2}\mathbb{E}\left\|\nabla F\left(w^t\right)\right\|^2-\frac{1}{2}\mathbb{E}\left\|\nabla F_{m(t)}\left(w^t\right)\right\|^2 \\
    &+\frac{1}{2}\mathbb{E}\left\|\nabla F\left(w^t\right)-\nabla F_{m(t)}\left(w^t\right)\right\|^2 
\end{align}

According to Assumption \ref{ass:heterogeneity}, it follows that
\begin{align}
	\label{prf5-2}
	\mathbb{E}\left\langle\nabla F\left(w^t\right),\nabla F\left(w^t\right)-\nabla F_{m(t)}\left(w^t\right)\right\rangle \le\frac{1}{2}\mathbb{E}\left\|\nabla F\left(w^t\right)\right\|^2-\frac{1}{2}\mathbb{E}\left\|\nabla F_{m(t)}\left(w^t\right)\right\|^2+\frac{1}{2}\sigma^2
\end{align}

This completes the proof.



\end{document}